\documentclass[12pt]{colt2014} % Anonymized submission
\usepackage{amssymb,amsmath,mathrsfs, bbm}
\usepackage{times}

\renewcommand{\vec}[1]{\boldsymbol{#1}}
\newcommand{\E}{\mathbb{E}}

\newcommand{\V}{\mathbb{V}}
\newcommand{\C}{\mathcal{C}}

\newcommand{\Hyp}{\mathcal{H}}
\newcommand{\F}{\mathcal{F}}
\newcommand{\R}{\mathbb{R}}
\newcommand{\Prob}{\mathbb{P}}

\newcommand{\Lu}{L_u}
\newcommand{\LN}{L_N}
\newcommand{\Lm}{\hat{L}_m}
\newcommand{\hm}{\hat{h}_m}
\newcommand{\Ex}{\mathcal{E}}

\newtheorem{assumptions}[theorem]{Assumptions}

\title[Localized Complexities for Transductive Learning]{Localized Complexities for Transductive Learning}
 \coltauthor{\Name{Ilya Tolstikhin} \Email{iliya.tolstikhin@gmail.com}\\
 \addr Computing Centre of Russian Academy of Sciences, Russia
\AND
 \Name{Gilles Blanchard} \Email{gilles.blanchard@math.uni-potsdam.de}\\
 \addr Department of Mathematics, University of Potsdam, Germany
  \AND
 \Name{Marius Kloft} \thanks{Most of the work was done while MK was with the Courant Institute of Mathematical Sciences and Memorial Sloan-Kettering Cancer Center, New York, NY, USA.} \Email{mkloft@cs.nyu.edu}\\
 \addr Department of Computer Science, Humboldt University of Berlin, Germany
 }
\begin{document}
\maketitle

\begin{abstract}
We show two novel concentration inequalities for suprema of empirical processes when sampling \emph{without} replacement, 
which both take the variance of the functions into account. While these inequalities may potentially have broad applications in learning theory in general,
we exemplify their significance by studying the \emph{transductive} setting of learning theory. For which we provide the first excess risk bounds
based on the localized complexity of the hypothesis class, which can yield fast rates of convergence
also in the transductive learning setting. We give a preliminary analysis 
of the localized complexities for the prominent case of kernel classes.
\end{abstract}

% old abstract
%In this paper, we consider the transductive setting of learning theory over a broad class of bounded and nonnegative loss functions. 
%We provide the first excess risk bound for the transductive learning based on the localized complexity of the hypothesis class, 
%which can yield fast rates of convergence.
%This is achieved by proving two new and very general concentration inequalities for suprema of empirical processes 
%when sampling without replacement, which have potential applications in learning theory also beyond the transductive setting.
%When applied to kernel classes, the transductive excess risk bound yields a fast rate of convergence.

\smallskip 


%======================================================================================================================
\section{Introduction}
%======================================================================================================================

The analysis of the stochastic behavior of empirical processes is a key ingredient in learning theory.
The supremum of empirical processes is of particular interest, playing a central role in various application areas, 
including the theory of empirical processes, VC theory, and Rademacher complexity theory, to name only a few.
Powerful Bennett-type concentration inequalities on the sup-norm of empirical processes introduced in \cite{Tal96a} (see also \cite{B02}) are at 
the heart of many recent advances in statistical learning theory, including \emph{local Rademacher complexities} \citep{BBM05,K11sf} and
related localization strategies \citep[cf.][Chapter 4]{Steinwart2008},
which can yield \emph{fast rates} of convergence on the excess risk.

These inequalities are based on the assumption of \emph{independent and identically distributed} random variables, 
commonly assumed in the inductive setting of learning theory and thus implicitly underlying many prevalent machine learning algorithms
such as support vector machines \citep{CorVap95,Steinwart2008}. 
However, in many cases the i.i.d. assumption breaks and substitutes for Talagrand's inequality are required.
For instance, the i.i.d. assumption is violated when training and test data come from different distributions or 
data points exhibit (e.g., temporal) interdependencies \cite[e.g.,][]{Steinwartjournal}.
Both scenarios are typical situations in visual recognition, computational biology, and many other application areas.

Another example where the i.i.d. assumption is void---in the focus of the present paper---is the \emph{transductive} setting of learning theory, 
where training examples are sampled independent and \emph{without} replacement 
from a finite population, instead of being sampled i.i.d. with replacement.
The learner in this case is provided with both a labeled training set and an unlabeled test set, and the goal is to predict the label of the test points.
This setting naturally appears in almost all popular application areas, including text mining, computational biology, recommender systems, visual
recognition, and computer malware detection, as effectively constraints are imposed on the samples, since they are inherently realized within the global
system of our world.
As an example, consider image categorization, which is an important task in the application area of visual recognition.
An object of study here could be the set of all images disseminated in the internet, only some of which are already reliably labeled 
(e.g., by manual inspection by a human), and the goal is to predict the unknown labels of the unlabeled images, in order to, e.g., make them 
accessible to search engines for content-based image retrieval.

From a theoretical view, however, the transductive learning setting is yet not fully understood. 
Several transductive error bounds were presented in series of works \citep{Vap82, Vap98, BL03, DEM04, CM06, EP09, CMP+09}, including the first analysis
based on \emph{global Rademacher complexities} presented in \cite{EP09}.
However, the theoretical analysis of the performance of transductive learning algorithms still remains less illuminated than in the classic inductive
setting: to the best of our knowledge, existing results do not provide fast rates of convergence in the general transductive setting.\footnote{
An exception are the works of \cite{BL03,CM06}, which consider, however, the case where the Bayes hypothesis has zero error and is 
contained in the hypothesis class. This is clearly an assumption too restrictive in practice, where the Bayes hypothesis usually cannot be assumed to be
contained in the class.
}

In this paper, we consider the transductive learning setting with arbitrary bounded nonnegative loss functions.
The main result is an excess risk bound for transductive learning based on the localized complexity of the hypothesis class.
This bound holds under general assumptions on the loss function and hypothesis class and can be viewed as a transductive analogue of Corollary 5.3 in 
\cite{BBM05}. The bound is very generally applicable with loss functions such as the squared loss and common hypothesis classes.
By exemplarily applying our bound to kernel classes, we achieve, for the first time in the transductive setup, an excess risk bound in terms
of the tailsum of the eigenvalues of the kernel, similar to the best known results in the inductive setting.
In addition, we also provide new transductive generalization error bounds that take the variances of the functions into account,
and thus can yield sharper estimates.

The localized excess risk bound is achieved by proving two novel 
\emph{concentration inequalities} for suprema of empirical processes when
sampling without replacement. The application of which goes far beyond the transductive learning setting---these concentration inequalities could
serve as a fundamental mathematical tool in proving results in various other areas of machine learning and learning theory. 
For instance, arguably the most prominent example in machine learning and learning theory of an empirical process where sampling without replacement is employed
is cross-validation \citep{Stone1974}, where training and test folds are sampled without replacement from the overall pool of examples, and
the new inequalities could help gaining a non-asymptotic understanding of cross-validation procedures.
However, the investigation of further applications of the novel concentration inequalities beyond the transductive learning setting
is outside of the scope of the present paper.

%======================================================================================================================
\section{The Transductive Learning Setting and State of the Art}
%======================================================================================================================

From a statistical point of view, the main difference between the transductive and inductive learning settings lies in the protocols used to obtain the training sample $S$.
Inductive learning assumes that the training sample is drawn i.i.d. from some fixed and unknown distribution $P$ on the product space $\mathcal{X}\times \mathcal{Y}$, 
where $\mathcal{X}$ is the input space and $\mathcal{Y}$ is the output space.
The learning algorithm chooses a predictor $h$ from some fixed hypothesis set $\mathcal{H}$ based on the training sample, and 
the goal is to minimize the true risk $\E_{X\times Y}[\ell\bigl( h(X),Y\bigr)]\to\min_{h\in\mathcal{H}}$ for a fixed, bounded, and 
nonnegative loss function $\ell\colon \mathcal{Y}^2\to[0,1]$.

We will use one of the two transductive settings\footnote{
The second setting assumes that both training and test sets are sampled i.\,i.\,d.\,from the same unknown distribution and the 
learner is provided with the labeled training and unlabeled test sets. It is pointed out by \cite{Vap98} that any upper bound on 
$\Lu(h) - \Lm(h)$ in the setting we consider directly implies a bound also for the second setting.}
considered in \cite{Vap98}, which is also used in \cite{DEM04,EP09}.  
Assume that a set $\vec X_N$ consisting of $N$ arbitrary input points is given (without any assumptions regarding its underlying source).
We then sample $m\leq N$ objects $\vec X_m\subseteq \vec X_N$ uniformly without replacement from $\vec X_N$ (which makes the inputs in $\vec X_m$ dependent).
Finally, for the input examples $\vec X_m$ we obtain their outputs $\vec Y_m$ by sampling, for each input $X\in\vec X_m$, the corresponding output $Y$ 
from some unknown distribution $P(Y|X)$. The resulting \emph{training set} is denoted by $S_m = (\vec X_m,\vec Y_m)$.
The remaining unlabeled set $\vec X_u = \vec X_N\setminus \vec X_m$, $u = N - m$ is the \emph{test set}.
Note that both \cite{DEM04} and \cite{EP09} consider a special case where the labels are obtained using some unknown but deterministic 
target function $\phi\colon \mathcal{X}\to\mathcal{Y}$ so that $P\bigl(\phi(X)|X\bigr) = 1$.
We will adopt the same assumption here.
The learner then chooses a predictor $h$ from some fixed hypothesis set $\mathcal{H}$ (not necessarily containing $\phi$) 
based on both the labeled training set $S_m$ and unlabeled test set~$\vec X_u$.
For convenience let us denote $\ell_h(X) =\ell\bigl(h(X),\phi(X)\bigr)$.
We define the test and training error, respectively, of hypothesis $h$ as follows: $L_u(h) = \frac{1}{u}\sum_{X\in \vec X_u} \ell_h(X)$ ,\:
$\hat{L}_m(h) = \frac{1}{m}\sum_{X\in \vec X_m} \ell_h(X)$,
where hat emphasizes the fact that the training (empirical) error can be computed from the data.
For technical reasons that will become clear later,
we also define the overall error of an hypothesis $h$ with regard to the union of the training and test sets as $L_N(h) = \frac{1}{N}\sum_{X \in \vec X_N} \ell_h(X)$
(this quantity will play a crucial role in the upcoming proofs).
Note that for a fixed hypothesis $h$ the quantity $\LN(h)$ is not random, as it is invariant under the partition into training and test sets.
The main goal of the learner in transductive setting is to select a hypothesis minimizing the test error $\Lu(h)\to\inf_{h\in\Hyp}$, 
which we will denote by $h^*_u$.

Since the labels of the test set examples are unknown, we cannot compute $\Lu(h)$ and need to estimate it based on the training sample $S_m$.
A common choice is to replace the test error minimization by \emph{empirical risk minimization} $\Lm(h)\to\min_{h\in\Hyp}$ and to use its solution,
which we denote by $\hm$, as an approximation of $h^*_u$.
For $h\in\Hyp$ let us define an \emph{excess risk} of $h$:
\[
\Ex_u(h) = \Lu(h) - \inf_{g\in\Hyp} \Lu(g) = \Lu(h) - \Lu(h^*_u). 
\]
A natural question is: how well does the hypothesis $\hm$ chosen by the ERM algorithm approximate the theoretical-optimal hypothesis $h^*_u$?
\smallskip

To this end, we use $\Ex_u(\hm)$ as a measure of the goodness of fit.
Obtaining tight upper bounds on $\Ex_u(\hm)$---so-called \emph{excess risk bounds}---is thus the main goal of this paper.
Another goal commonly considered in learning literature is the one of obtaining upper bounds on $\Lu(\hm)$ in terms of $\Lm(\hm)$, 
which measures the generalization performance of empirical risk minimization.
Such bounds are known as the \emph{generalization error bounds}.
Note that both $\hm$ and $h^*_u$ are random, since they depend on the training and test sets, respectively.
Note, moreover, that for any fixed $h\in\Hyp$ its excess risk $\Ex_u(h)$ is also random.
Thus both tasks (of obtaining excess risk and generalization bounds, respectively) deal with random quantities and require bounds that hold with high probability.

The most common way to obtain generalization error bounds for $\hm$ is to introduce uniform deviations over the class $\Hyp$:
\begin{equation}
\label{eq:firstsupnorm}
\Lu(\hm) - \Lm(\hm) \leq
\sup_{h\in\Hyp}\Lu(h) - \Lm(h).
\end{equation}
The random variable appearing on the right side is directly related to the sup-norm of the empirical process \citep{BLM13}.
It should be clear that, in order to analyze the transductive setting, it is of fundamental importance to obtain 
high-probability bounds for functions $f(Z_1,\dots, Z_m)$, where $\{Z_1,\dots,Z_m\}$ are random variables sampled \emph{without} 
replacement from some fixed finite set.
Of particular interest are concentration inequalities for sup-norms of empirical processes, which we present in Section \ref{sec:Concentr}.

\subsection{State of the Art and Related Work}
Error bounds for transductive learning were considered by several authors in recent years.
Here we name only a few of them\footnote{
For an extensive review of transductive error bounds we refer to \cite{P08}.
}.
The first general bound for binary loss functions, presented in \cite{Vap82}, was \emph{implicit} in the sense that the value of the bound was 
specified as an outcome of a computational procedure.
The somewhat refined version of this implicit bound also appears in \cite{BL03}.
It is well known that generalization error bounds with fast rates of convergence can be obtained under certain restrictive assumptions on the problem at hand.
For instance, \cite{BL03} provide a bound that has an order of $\frac{1}{\min(u,m)}$ in the \emph{realizable} case, i.e., when $\phi \in \Hyp$ 
(meaning that the hypothesis having zero error belongs to $\Hyp$).
However, such an assumption is usually unrealistic: in practice it is usually impossible to avoid overfitting when choosing $\Hyp$ so large that
it contains the Bayes classifier.

The authors of \cite{CM06} consider a transductive regression problem with bounded squared loss and obtain a generalization error bound of the 
order $\sqrt{\Lm(\hm) \frac{\log N}{\min(m,u)}}$, which also does not attain a fast rate.
Several PAC-Bayesian bounds were presented in \cite{BL03, DEM04} and others. 
However their tightness critically depends on the \emph{prior distribution} over the hypothesis class, which should be fixed by the learner prior to observing the training sample.
Transductive bounds based on algorithmic stability were presented for classification in \cite{EP06} and for regression in \cite{CMP+09}.
However both are roughly of the order $\min(u,m)^{-1/2}$.
Finally, we mention the results of \cite{EP09} based on transductive Rademacher complexities.
However, the analysis was based on the \emph{global} Rademacher complexity combined with a McDiarmid-style 
concentration inequality for sampling without replacement and thus does not yield fast convergence rates.

%======================================================================================================================
\section{Novel Concentration Inequalities for Sampling Without replacement}\label{sec:Concentr}
%======================================================================================================================

In this section, we present two new \emph{concentration inequalities} for suprema of empirical processes when sampling without replacement.
The first one is a sub-Gaussian inequality that is based on a result by \cite{B04} and closely related to the \emph{entropy method} \citep{BLM13}.
The second inequality is an analogue of Bousquet's version of Talagrand's concentration inequality \citep{Bphd02, B02, Tal96a} 
and is based on the reduction method first suggested in \cite{Hoe63}. 

Next we state the setting and introduce the necessary notation. Let $\C = \{c_1,\dots,c_N\}$ be some finite set.
For $m\leq N$ let $\{Z_1,\dots,Z_m\}$ and $\{X_1,\dots,X_m\}$ be sequences of random variables sampled uniformly without and with replacement from $\C$, respectively.
Let $\F$ be a 
(countable\footnote{
Note that all results can be translated to the uncountable classes, for instance, if the empirical process is \emph{separable}, meaning that $\F$ contains a dense countable subset. 
We refer to page 314 of \cite{BLM13} or page 72 of \cite{Bphd02}.
}
) 
class of functions $f\colon \C\to\R$, such that $\E[f(Z_1)] = 0$ and $f(x) \in [-1, 1]$ for all $f\in\F$ and $x\in\C$.
It follows that $\E[f(X_1)] = 0$ since $Z_1$ and $X_1$ are identically distributed.
Define the variance $\sigma^2 = \sup_{f\in\F}\V[f(Z_1)]$.
Note that $\sigma^2=\sup_{f\in\F}\E[f^2(Z_1)]=\sup_{f\in\F}\V[f(X_1)]$.
%\[
%\sigma^2 = \sup_{f\in\F}\E[f^2(Z_1)] = \sup_{f\in\F}\V[f(Z_1)] = \sup_{f\in\F}\V[f(X_1)] = \sup_{f\in\F}\E[f^2(X_1)].
%\]
Let us define the \emph{supremum of the empirical process} for sampling with and without replacement, respectively:\footnote{
The results presented in this section can be also generalized to $\sup_{f\in\F} \left|\sum_{i=1}^m f(Z_i)\right|$ using the same techniques.
}
\[
Q_m = \sup_{f\in\F} \sum_{i=1}^m f(X_i),\quad 
Q'_m = \sup_{f\in\F} \sum_{i=1}^m f(Z_i).
\]
The random variable $Q_m$ is well studied in the literature and there are remarkable Bennett-type concentration inequalities for $Q_m$, 
including Talagrand's inequality \citep{Tal96a} and its versions due to \cite{B02, Bphd02} and others.\footnote{
For completeness we present one such inequality for $Q_m$ as Theorem~\ref{thm:Bphd02} in Appendix~\ref{appendix:Bousquet}.
For the detailed review of concentration inequalities for $Q_m$ we refer to Section 12 of \cite{BLM13}.
}
The random variable $Q'_m$, on the other hand, is much less understood, and no Bennett-style concentration inequalities are known for it up to date.

%===========================================================
\subsection{The New Concentration Inequalities}
%===========================================================

In this section, we address the lack of Bennett-type concentration inequalities for $Q'_m$
by presenting two novel inequalities for suprema of empirical processes when sampling without replacement.
\begin{theorem}[\rm Sub-Gaussian concentration inequality for sampling \emph{without} replacement]
\label{thm:SupConc1}
For any ${\epsilon \geq 0}$,
\begin{equation}
\label{eq:SupConc1}
\Prob\left\{ Q'_m - \E[Q'_m] \geq \epsilon\right\}
\leq
\exp\left\{
-\frac{(N+2)\epsilon^2}{8N^2\sigma^2}
\right\}
<
\exp\left\{
-\frac{\epsilon^2}{8N\sigma^2}
\right\}.
\end{equation}
The same bound also holds for $\Prob\left\{ \E[Q'_m]  - Q'_m\geq \epsilon\right\}$.
Also for all $t \geq 0$ the following holds with probability greater than $1 - e^{-t}$:
\begin{equation}
\label{eq:SupDev1}
Q'_m \leq \E[Q'_m] + 2\sqrt{2N \sigma^2t}.
\end{equation}
\end{theorem}

 %\comment{I think the reasoning should go like this: we need to show that if $\ell_h - \ell_{h^*}$ is in $F^*$ than for any $a\in[0,1]$ the following: $a(\ell_h - \ell_{h^*})$ also lies in $F^*$. We are going to find $h' \in H$ such that $a(\ell_h - \ell_{h^*}) = \ell_{h'} - \ell_{h^*}$. Or equivalently $\ell_{h'} = a \ell_h + (1-a) \ell_{h^*}$. I suspect that if $\ell$ is continuous and H is convex there is such an $h'$ by definition... Can you point me to the precise theorem from calculus? ;)}
%
%\bigskip
%
%\comment{You are saying, we need to show that for any continuous loss $\ell$ and any $a\in[0,1]$, any h and the optimal h*, there is a h' in H such that it holds: $\ell_{h'} = a \ell_h + (1-a) \ell_{h^*}$ }
%\bigskip 
%
%\comment{
%\[
%F(angel) = \ell( h^*(X) + a ( h(X) - h^*(X) ) ).
%\]
%It is continuous on $a\in[0,1]$ since $\ell$ is continous. And the argument of $\ell$(..) runs only in H. 
%For $a\in[0,1]$ the range of F(angel) contains $\ell( h(X) )$ and $\ell( h^*(X) )$. F is continuous, thus it contains the whole range $\ell( h^*(X) ), \ell( h(X) )$.
%It means that for any $a\in[0,1]$ there is $h' \in H$ such that 
%$F(angel) = \ell( h'(X) )$
%}

\begin{theorem}[\rm Talagrand-type concentration inequality for sampling \emph{without} replacement]
\label{thm:SupConc2}
Define $v = m\sigma^2 + 2 \E[Q_m]$.
For $u> -1$ define $\phi(u) = e^u - u - 1$,\: $h(u) = (1+u)\log(1+u) - u$.
Then for any~$\epsilon\geq  0$:
\begin{align*}
\Prob\left\{
Q_m' - \E[Q_m] \geq \epsilon
\right\}
&\leq
\exp\left(
-v h\left(
\frac{\epsilon}{v}
\right)
\right)
\leq
\exp\left(-
\frac{\epsilon^2}{2v + \frac{2}{3}\epsilon}
\right).
\end{align*}
Also for any  $t\geq 0$ following holds with probability  greater than $1-e^{-t}$:
\[
Q_m' \leq \E[Q_m] + \sqrt{2vt} + \frac{t}{3}.
\]
\end{theorem}
The appearance of $\E[Q_m]$ in the last theorem might seem unexpected on the first view. 
Indeed,  one usually wants to control the concentration of a random variable around its expectation. 
However, it is shown in the lemma below that in many cases $\E[Q_m]$ will be close to $\E[Q'_m]$:
\begin{lemma}
\label{lemma:expect}
\[
0 \leq
\E[Q_m] - \E[Q'_m]
\leq
2\frac{m^3}{N}.
\]
\end{lemma}
The above lemma is proved in Appendix~\ref{appendix:Proofs}.
It shows that for $m = o(N^{2/5})$ the order of $\E[Q_m] - \E[Q'_m]$ does not exceed $\sqrt{m}$, 
and thus Theorem \ref{thm:SupConc2} can be used to control the deviations of $Q'_m$ above its expectation $\E[Q'_m]$ at a fast rate.
However, generally $\E[Q'_m]$ could be smaller than $\E[Q_m]$,  which may potentially lead to significant gap, in which case
Theorem~\ref{thm:SupConc1} is the preferred choice to control the deviations of $Q_m'$ around $\E[Q_m']$.

%===========================================================
\subsection{Discussion}\label{subsec:discussion}
%===========================================================

It is worth comparing the two novel inequalities for $Q'_m$ to the best known results in the literature.
To this end, we compare our inequalities with the McDiarmid-style inequality recently obtained in \cite{EP09} 
(and slightly improved in \cite{CMP+09}):
\begin{theorem}[\cite{EP09}\footnotemark]
\label{thm:EP}
For all $\epsilon \geq 0$:
\begin{equation}
\label{eq:EP}
\Prob\left\{
Q_m' - \E[Q_m'] \geq \epsilon
\right\}
\leq
\exp\left\{
-\frac{\epsilon^2}{2m}\left(\frac{N-1/2}{N-m}\right)\left(1 - \frac{1}{2\max(m,N-m)}\right)
\right\}.
\end{equation}
The same bound also holds for $\Prob\left\{\E[Q'_m] - Q'_m \geq \epsilon\right\}$.
\end{theorem}
\footnotetext{
This bound does not appear explicitly in \cite{EP09,CMP+09}, but can be immediately obtained using Lemma 2 of \cite{EP09} for $Q'_m$ with $\beta = 2$.
}
%This inequality was used in \citep{EP09} to derive (global) Rademacher bounds for the transductive setup.

To begin with, let us notice that Theorem~\ref{thm:EP} does not account for the variance $\sup_{f\in\F}\V[f(X_1)]$, while Theorems~\ref{thm:SupConc1} and Theorem~\ref{thm:SupConc2} do. 
As it will turn out in Section~\ref{sec:fastRates}, this refined treatment of the variance $\sigma^2$ allows us to use localization techniques,
facilitating to obtain sharp estimates (and potentially, fast rates) also in the transductive learning setup. 
The comparison between concentration inequalities \eqref{eq:SupConc1} of Theorem~\ref{thm:SupConc2} and~\eqref{eq:EP} of Theorem~\ref{thm:EP} is as follows:
note that the term $1 - \frac{1}{2\max(m,N-m)}$ is negligible for large $N$, so that slightly re-writing the inequalities boils down to comparing
$
-\frac{\epsilon^2}{8m\sigma^2}\frac{m}{N} 
$
and
$
-\frac{\epsilon^2}{2m}\left(\frac{N-1/2}{N-m}\right).
$
For $m = o(N)$ (which in a way transforms sampling without replacement to sampling with replacement), the second inequality clearly outperforms the first one.
However, for the case when $m=\Omega(N)$ (frequently used in the transductive setting), say $N = 2m$, the comparison depends on the relation between 
$-\epsilon^2/(16 \,m \sigma^2)$ and $-\epsilon^2/m$ and the result of Theorem \ref{thm:SupConc1} outperforms the one of El-Yaniv and Pechyony for $\sigma^2 \leq 1/16$.
The comparison between Theorems \ref{thm:SupConc2} and \ref{thm:EP} for both cases ($m=o(N)$ and $m=\Omega(N)$) depends on the value of $\sigma^2$.

Theorem \ref{thm:SupConc2} is a direct analogue of Bousquet's version of Talagrand's inequality (see Theorem~\ref{thm:Bphd02} in Appendix~\ref{appendix:Bousquet} of the supplementary material), 
frequently used in the learning literature.
It states that the upper bound on $Q_m$, provided by Bousquet's inequality, also holds for $Q'_m$.
Now we compare Theorems~\ref{thm:SupConc1} and \ref{thm:SupConc2}.
First of all note that the deviation bound \eqref{eq:SupDev1} does not have the term $2\E[Q_m]\geq 0$ under the square root in contrast to Theorem \ref{thm:SupConc2}.
As will be shown later, in some cases this fact can result in improved constants when applying Theorem \ref{thm:SupConc1}.
Another nice thing about Theorem \ref{thm:SupConc1} is that it provides upper bounds for both $Q'_m - \E[Q'_m]$ and $\E[Q'_m] - Q'_m$, while Theorem~\ref{thm:SupConc2} upper bounds only $Q'_m - \E[Q'_m]$.
The main drawback of Theorem \ref{thm:SupConc1} is the factor $N$ appearing in the exponent.
Later we will see that in some cases it is more preferable to use Theorem \ref{thm:SupConc2} because of this fact.
%(thus Theorem \ref{thm:SupConc2} opens an obvious way to generalize all the results of local rademacher analysis applied to the inductive setting on transductive setting).

We also note that, if $m=\Omega(N)$ or $m=o(N^{2/5})$, we can control the deviations of $Q'_m$ around $\E[Q'_m]$ with 
inequalities that are similar to i.i.d.\,case.  It is an open question, however, whether this can be done also for other regimes of $m$ and $N$.
It should be clear though that we can obtain at least as good rates as in the inductive setting using Theorem~\ref{thm:SupConc2}.
To summarize the discussion, when $N$ is large and $m=o(N)$, Theorems \ref{thm:SupConc2} or \ref{thm:EP} 
(depending on $\sigma^2$ and the order of $\E[Q_m] - \E[Q'_m]$) can be significantly tighter than Theorem \ref{thm:SupConc1}.
However, if $m=\Omega(N)$, Theorem \ref{thm:SupConc1} is more preferable. 
Further discussions are presented in Appendix \ref{sec:FurtherDiscussions}.

%===========================================================
\subsection{Proof Sketch}
%===========================================================

Here we briefly outline the proofs of Theorems \ref{thm:SupConc1} and \ref{thm:SupConc2}.
Detailed proofs are given in Appendix~\ref{appendix:Proofs} of the supplementary material.

Theorem \ref{thm:SupConc2} is a direct consequence of Bousquet's inequality and Hoefding's reduction method.
It was shown in Theorem 4 of \cite{Hoe63} that, for any convex function $f$, the following inequality holds: 
\[
\E\left[f\left(\sum_{i=1}^m Z_i\right)\right]\leq \E\left[f\left(\sum_{i=1}^m X_i\right)\right].
\]
Although not stated explicitly in \cite{Hoe63}, the same result also holds if we sample from finite set of vectors instead of real numbers \citep{GN10}.
This reduction to the i.i.d. setting together with some minor technical results is enough to bound the moment generating function of $Q'_m$ and obtain a concentration inequality using Chernoff's method (for which we refer to the Section 2.2 of \cite{BLM13}).

The proof of Theorem \ref{thm:SupConc1} is more involved.
It is based on the sub-Gaussian inequality presented in Theorem 2.1 of \cite{B04}, which is related to the \emph{entropy method} introduced by M.\,Ledoux (see \cite{BLM13} for references).
Consider a function $g$ defined on the partitions $X^N = (X^m\cup X^u)$ of a fixed finite set $X^N$ of cardinality $N$ into two disjoint subsets $X^m$ and $X^u$ of cardinalities $m$ and $u$, respectively, where $N=m+u$.
Bobkov's inequality states that, roughly speaking, if $g$ is such that the Euclidean length of its discrete gradient $|\nabla g(X^m\cup X^u)|^2$ is bounded by a constant $\Sigma^2$, and if the
partitions $(X^m\cup X^u)$ are sampled uniformly from the set of all such partitions, then $g(X^m\cup X^u)$ is sub-Gaussian with parameter $\Sigma^2$.

%===========================================================
\subsection{Applications of the New Inequalities}
%===========================================================

The novel concentration inequalities presented above can be generally used as a mathematical tool in various areas of machine learning
and learning theory where suprema of empirical processes over sampling without replacement are of interest, including the analysis of
cross-validation and low-rank matrix factorization procedures as well as the transductive learning setting. 
Exemplifying their applications, we show in the next section---for the first time in the transductive setting of learning theory---excess risk bounds in terms of localized
complexity measures, which can yield sharper estimates than global complexities.
%This exemplifies that the concentration inequalities can be used to obtain sharp performance estimates for
%algorithmic procedures such as ERM.

%======================================================================================================================
\section{Excess Risk Bounds for Transductive Learning via Localized Complexities}\label{sec:fastRates}
%======================================================================================================================

We start with some preliminary generalization error bounds that show how to apply the concentration inequalities of 
Section~\ref{sec:Concentr} to obtain risk bounds in the transductive learning setting.
Note that \eqref{eq:firstsupnorm} can be written in the following way:
\begin{align*}
\Lu(\hm) - \Lm(\hm) \leq
\sup_{h\in\Hyp}\Lu(h) - \Lm(h) = \frac{N}{u}\cdot\sup_{h\in\Hyp}\LN(h) - \Lm(h),
\end{align*}
where we used the fact that $N \cdot L_N(h) = m \cdot \hat{L}_m(h) + u \cdot L_u(h)$.
Note that for any fixed $h\in\Hyp$, we have
$
\LN(h) - \Lm(h)
=
\frac{1}{m}\sum_{X\in\vec X_m}\bigl( \LN(h) - \ell_h(X)\bigr),
$
where $\vec X_m$ is sampled uniformly without replacement from $\vec X_N$.
Note that we clearly have $\LN(h) - \ell_h(X)\in[-1,1]$ and $\E[\LN(h) - \ell_h(X)] = \LN(h) - \E[\ell_h(X)] =0$.
Thus we can use the setting described in Section \ref{sec:Concentr}, with $\vec X_N$ playing the role of $\mathcal{C}$ and 
considering the function class $\F_{\Hyp} =  \{f_h\colon f_h(X) = \LN(h) - \ell_h(X), h\in\Hyp \}$ associated with $\Hyp$, to obtain high-probability bounds for 
$
\sup_{f_h\in\F_{\Hyp}} \sum_{X\in\vec X_m} f_h(X)=m \cdot\sup_{h\in\Hyp} \bigl( \LN(h) - \Lm(h) \bigr).
$
Note that in Section \ref{sec:Concentr} we considered unnormalized sums, hence we obtain a factor of $m$ in the above equation.
As already noted, for fixed $h$, $L_N(h)$ is not random;
also centering random variable does not change its variance. 
Keeping this in mind, we define
\begin{equation}
\label{eq:supvar}
\sigma_{\Hyp}^2 
=
 \sup_{f_h\in\F_{\Hyp}} \V[f_h(X)]
=
 \sup_{h\in\Hyp} \V[\ell_h(X)]
 =
 \sup_{h\in\Hyp} \left(\frac{1}{N}\sum_{X\in\vec X_N}\bigl(\ell_h(X) - \LN(h) \bigr)^2\right).
\end{equation}

Using Theorems \ref{thm:SupConc1} and \ref{thm:SupConc2}, we can obtain the following results that hold \emph{without any assumptions} on the learning problem at hand, 
except for the boundedness of the loss function in the interval $[0,1]$.
Our first result of this section follows immediately from the new concentration inequality of Theorem \ref{thm:SupConc1}:
\begin{theorem}
\label{thm:eb1}
For any $t\geq 0$ with probability greater than $1-e^{-t}$ the following holds:
\begin{align*}
\forall h\in \Hyp:\quad
%\sup_{h\in\Hyp}\bigl(\LN(h) - \Lm(h)\bigr)
\LN(h) - \Lm(h)
&\leq 
\E\left[\sup_{h\in\Hyp} \bigl( \LN(h) - \Lm(h) \bigr)\right] + 
2\sqrt{2\left(\frac{N}{m^2}\right) \sigma_{\Hyp}^2t},%\\
%&\leq
%\frac{u}{N} \E\left[R_{m,n}(\Hyp)\right] + 2\sqrt{2\left(\frac{N}{m^2}\right) \sigma_{\Hyp}^2t},
\end{align*}
where $\sigma_{\Hyp}^2$ was defined in \eqref{eq:supvar}.
\end{theorem}
Let $\{\xi_1,\dots,\xi_m\}$ be random variables sampled \emph{with replacement} from $\vec X_N$ and denote
\[
E_m = \E\left[\sup_{h\in\Hyp} \left( \LN(h) - \frac{1}{m}\sum_{i=1}^m \ell_{h}(\xi_i) \right)\right].
\]
The following result follows from Theorem \ref{thm:SupConc2} by simple calculus.
We provide the detailed proof in the supplementary material.
\begin{theorem}
\label{thm:eb2}
For any $t\geq 0$ with probability greater than $1-e^{-t}$ the following holds:
\begin{align*}
\forall h\in \Hyp:\quad
%\sup_{h\in\Hyp}\bigl(\LN(h) - \Lm(h) \bigr)
\LN(h) - \Lm(h)
&\leq 
2E_m + \sqrt{\frac{2\sigma^2_{\Hyp} t}{m}} + \frac{4t}{3m},
\end{align*}
where $\sigma_{\Hyp}^2$ was defined in \eqref{eq:supvar}.
%Note that $E_m$ can be bounded from above using twice the standard inductive rademacher complexity.
\end{theorem}
\begin{remark}
Note that $E_m$ is an expected sup-norm of the empirical process naturally appearing in inductive learning.
Using the well-known \emph{symmetrization inequality} (see Section 11.3 of \cite{BLM13}), we can upper bound it by 
twice the expected value of the supremum of the Rademacher process.
In this case, the last theorem thus gives exactly the same upper bound on the quantity $\sup_{h\in\Hyp}\bigl(\LN(h) - \Lm(h) \bigr)$ as the one of 
Theorem 2.1 of \cite{BBM05} (with $\alpha = 1$ and $(b - a) = 1$).
\end{remark}
Here we provide some discussion on the two generalization error bounds presented above.
Note that $\sigma^2_{\Hyp}\leq 1/4$, since $\sigma^2_{\Hyp}$ is the variance of a random variable bounded in the interval $[0,1]$.
We conclude that the bound of Theorem \ref{thm:eb2} is of the order $m^{-1/2}$, since the typical order\footnote{
For instance if $\mathcal{F}$ is finite it follows from Theorems 2.1 and 3.5 of \cite{K11}.
}
of $E_m$ is also $m^{-1/2}$.
Note that repeating the proof of Lemma \ref{lemma:expect} we immediately obtain the following corollary:
\begin{corollary}
\label{corr:ExWandWO}
Let $\{\xi_1,\dots,\xi_m\}$ be random variables sampled \emph{with replacement} from $\vec X_N$.
For any countable set of functions $\F$ defined on $\vec X_N$ the following holds: 
\[
\E\left[\sup_{f\in\F}  \E[f(X)] - \frac{1}{m}\sum_{X\in\vec X_m}f(X) \right] \leq 
\E\left[\sup_{f\in\F}  \E[f(X)] - \frac{1}{m}\sum_{i=1}^mf(\xi_i) \right].
\]
%\[
%\E\left[\sup_{h\in\Hyp} \bigl( \LN(h) - \Lm(h) \bigr)\right] \leq E_m.
%\]
\end{corollary}
The corollary shows that for $m=\Omega(N)$ the bound of Theorem \ref{thm:eb1} also has the order $m^{-1/2}$.
However, if $m=o(N)$, the convergence becomes slower and it can even diverge for $m=o(N^{1/2})$.
\begin{remark}
\label{remark:WhyBad}
The last corollary enables us to use also in the transductive setting all the established techniques related to the inductive Rademacher process, 
including symmetrization and contraction inequalities. Later in this section, we will employ this result to derive excess risk bounds for kernel classes
in terms of the tailsum of the eigenvalues of the kernel, which can yield a fast rate of convergence.
However, we should keep in mind that there might be a significant gap between $\E[Q_m]$ and $\E[Q'_m]$, in which case such a reduction can be loose.
\end{remark}

%===========================================================
\subsection{Excess Risk Bounds}
%===========================================================

The main goal of this section is to analyze to what extent the known results on localized risk bounds presented in series of 
works \citep{KP99,M00,BBM05, K06} can be generalized to the transductive learning setting.
These results essentially show that the rate of convergence of the excess risk is related to the fixed point of the modulus of 
continuity of the empirical process associated with the hypothesis class. Our main tools to this end will be the sub-Gaussian and 
Bennett-style concentration inequalities of Theorems \ref{thm:SupConc1} and \ref{thm:SupConc2} presented in the previous section.

From now on it will be convenient to introduce the following operators, mapping functions $f$ defined on $\vec X_N$ to $\R$:
\[
E f = \frac{1}{N}\sum_{X\in\vec X_N} f(X),
\quad
\hat{E}_m f = \frac{1}{m}\sum_{X\in\vec X_m} f(X).
\]
Using this notation we have: $\LN(h) = E \ell_h$ and $\Lm(h) = \hat{E}_m \ell_h$.

Define the \emph{excess loss class} $\F^* = \{\ell_h - \ell_{h^*_N}, h\in\Hyp\}$.
Throughout this section, we will assume that the loss function $\ell$ and hypothesis class $\Hyp$ satisfy the following assumptions:
\begin{assumptions}
\label{ass:Excess}
\begin{enumerate}
%\item
%Loss function $\ell$ is L-Lipschitz in its first argument: for all $y_1,y_2,y$
%\[
%|\ell(y_1,y) - \ell(y_2,y)| \leq L|y_1 - y_2|.
%\]
\item
There is a function $h^*_N\in\Hyp$ satisfying $\LN(h^*_N) = \inf_{h\in\Hyp}\LN(h)$.
%\item
%The \emph{excess loss class} $\F^* = \{\ell_h - \ell_{h^*_N}, h\in\Hyp\}$ is \emph{star-shaped} around zero, meaning that, for any $f\in\F^*$ and $\alpha\in[0,1]$, it holds $\alpha f\in\F^*$ .
\item
There is a constant $B > 0$  such that for every $f\in\F^*$ we have $E f^2 \leq B\cdot Ef$.
%\[
%\frac{1}{N}\sum_{X\in\vec X_N} \left(\ell_h(X) - \ell_{h^*_N}(X)\right)^2 = 
%E f^2 \leq B\cdot Ef
%=
%B\bigl( \LN(h) - \LN(h^*_N)\bigr).
%\]
\item
As before the loss function $\ell$ is bounded in the interval $[0,1]$.
\end{enumerate}
\end{assumptions}
Here we shortly discuss these assumptions.
Assumption \ref{ass:Excess}.1 is quite common and not restrictive.
%Assumption \ref{ass:Excess}.2 \comment{probably holds if $\Hyp$ is convex and $\ell$ is continuous}.
%Note that we can drop this assumption in several ways.
%First way is by using the peeling technique described in the first part of the proof of Theorem 3.3 in \cite{BBM05}.
%Another way is to replace $F^*$ in \eqref{eq:subRoor} with its star-hull: $\{\alpha f\colon f\in\F^*, \alpha\in[0,1]\}$.
Assumption~\ref{ass:Excess}.2 can be satisfied, for instance, when the loss function $\ell$ is Lipschitz and there is a constant $T>1$ such that $\frac{1}{N}\sum_{X\in\vec X_N} \bigl(h(X) - h^*_N(X) \bigr)^2 \leq T \bigl( \LN(h) - \LN(h^*_N)\bigr)$ for all $h\in\Hyp$.
These conditions are satisfied for example for the quadratic loss $\ell(y,y')=(y-y')^2$ with uniformly bounded convex classes $\Hyp$ (for other examples we refer to Section 5.2 of \cite{BBM05} and Section 2.1 of \cite{BMP10}).
%Assumption \ref{ass:Excess}.3 was made throughout the paper. 
Assumption \ref{ass:Excess}.3 could be possibly relaxed using some analogues of Theorems \ref{thm:SupConc1} and \ref{thm:SupConc2} that hold for classes $\F$ 
with unbounded functions\footnote{
\cite{A08} show a version of Talagrand's inequality for unbounded functions in the i.i.d. case.
}.

Next we present the main results of this section,
which can be considered as an analogues of Corollary~5.3 of \cite{BBM05}.
The results come in pairs, depending on whether Theorem \ref{thm:SupConc1} or \ref{thm:SupConc2} is used in the proof.
We will need the notion of a \emph{sub-root} function, which is a nondecreasing and nonnegative function $\psi\colon[0,\infty)\to[0,\infty)$, 
such that $r\to\psi(r)/\sqrt{r}$ is nonincreasing for $r>0$.
It can be shown that any sub-root function has a unique and positive fixed point.
\begin{theorem}
\label{thm:LocExBound}
Let $\Hyp$ and $\ell$ be such that Assumptions~\ref{ass:Excess} are satisfied.
Assume there is a sub-root function $\psi_m(r)$ such that 
\vspace{-7pt}
\begin{equation}
\label{eq:subRoor}
B\, \E\left[ \sup_{f\in\F^*: E f^2 \leq r} (E f - \hat{E}_m f)\right] \leq \psi_m(r).
\end{equation}
Let $r_m^*$ be a fixed point of $\psi_m(r)$.
Then for any $t>0$ with probability greater than $1-e^{-t}$ we have:
\[
\LN(\hm) - \LN(h^*_N) \leq 
51\frac{r^*_m}{B}  + 17  B t\left(\frac{N}{m^2}\right).
\]
\end{theorem}
We emphasize that the constants appearing in Theorem \ref{thm:LocExBound} are slightly better than the ones appearing in Corollary 5.3 of \cite{BBM05}.
This result is based on Theorem \ref{thm:SupConc1} and thus shares the disadvantages of Theorem \ref{thm:eb1} discussed above: 
the bound does not converge for $m=o(N^{-1/2})$.
However, by using Theorem~\ref{thm:SupConc2} instead of Theorem~\ref{thm:SupConc1} in the proof, we can replace the factor of $N/m^2$ appearing in the bound by a factor of $1/m$ at a price of slightly worse constants:
%Further discussions are presented in Remark \ref{remark:Important}.
\begin{theorem}
\label{thm:LocExBound_Mod}
Let $\Hyp$ and $\ell$ be such that Assumptions~\ref{ass:Excess} are satisfied.
Let $\{\xi_1,\dots,\xi_m\}$ be random variables sampled \emph{with replacement} from $\vec X_N$.
Assume there is a sub-root function $\psi_m(r)$ such that 
\begin{equation}
\label{eq:subRoor2}
B\cdot \E\left[ \sup_{f\in\F^*: E f^2 \leq r} \left(E f - \frac{1}{m}\sum_{i=1}^m f(\xi_i)\right)\right] \leq \psi_m(r).
\end{equation}
Let $r_m^*$ be a fixed point of $\psi_m(r)$.
Then for any $t>0$, with probability greater than $1-e^{-t}$, we have:
\[
\LN(\hm) - \LN(h^*_N) \leq 
901\frac{r_m^*}{B} + \frac{t(16 + 25B)}{3m}.
\]
\end{theorem}
We also note that in Theorem \ref{thm:LocExBound_Mod} the modulus of continuity of the empirical process over sampling without replacement appearing in the left-hand side of~\eqref{eq:subRoor} is replaced with its inductive analogue.
As follows from Corollary \ref{corr:ExWandWO}, the fixed point $r^*_m$ of Theorem \ref{thm:LocExBound} can be smaller than that of Theorem~\ref{thm:LocExBound_Mod} and thus, for large $N$ and $m=\Omega(N)$ the first bound can be tighter.
Otherwise, if $m=o(N)$, Theorem \ref{thm:LocExBound_Mod} can be more preferable. 

\medskip
\noindent
{\bf Proof sketch:}
Now we briefly outline the proof of Theorem \ref{thm:LocExBound}.
%The detailed proofs are presented in Section \ref{sec:Proofs2} of the supplementary material.
It is based on the \emph{peeling technique} and consists of the steps described below (similar to the proof of the first part of Theorem 3.3 in \cite{BBM05}).
The proof of Theorem \ref{thm:LocExBound_Mod} repeats the same steps with the only difference being that Theorem \ref{thm:SupConc2} is used on Step 2 instead of Theorem \ref{thm:SupConc1}.
The detailed proofs are presented in Section \ref{sec:Proofs2} of the supplementary material.

\paragraph{\rm\textsc{Step 1}} First we fix an arbitrary $r>0$ and introduce the rescaled version of the centered loss class:
$\mathcal{G}_r = \left\{\frac{r}{\Delta(r,f)}f,\; f\in\F^*\right\}$,
where $\Delta(r,f)$ is chosen such that the variances of the functions contained in $\mathcal{G}_r$ do not exceed $r$.

\paragraph{\rm\textsc{Step 2}} 
We can use Theorem \ref{thm:SupConc1} to obtain the following upper bound on $V_r=\sup_{g\in\mathcal{G}_r}( E g - \hat{E}_m g)$ which holds with probability greater than $1-e^{-t}$:
$V_r\leq \E[V_r] + 2\sqrt{2t\left(\frac{N}{m^2}\right) r}$.

\paragraph{\rm\textsc{Step 3}} 
Using the \emph{peeling technique} (which consists in dividing the class $\F^*$ into slices of functions having variances within a certain range), 
we are able to show that $\E[V_r]\leq 5\psi_m(r)/B$.
Also, using the definition of sub-root functions, we conclude that  $\psi_m(r)\leq\sqrt{rr^*_m}$ for any $r\geq r^*_m$,
which gives us
$ V_r\leq \sqrt{r} \left(5\frac{\sqrt{r^*_m}}{B} + 2\sqrt{2t\left(\frac{N}{m^2}\right) }\right)$.

\paragraph{\rm\textsc{Step 4}} Now we can show that by properly choosing $r_0>r^*_m$  we can get that, for any $K>1$, it holds $V_{r_0} \leq \frac{r_0}{KB}$.
Using the definition of $V_r$, we obtain that the following holds, with probability greater than $1-e^{-t}$:
\vspace{-5pt}
\begin{align*}
\forall f\in\F^*, \forall K > 1:\quad
E f - \hat{E}_m f 
&\leq  
\frac{\max\bigl(r_0,E f^2\bigr)}{r_0} \frac{r_0}{KB}
=
\frac{\max\bigl(r_0,E f^2\bigr)}{KB}.
\end{align*}

\paragraph{\rm\textsc{Step 5}} Finally it remains to upper bound $Ef$ for the two cases $E f^2 > r_0$ and $ Ef^2 \leq r_0$ 
(which can be done using Assumption~\ref{ass:Excess}.2), to combine those two results, and to 
notice that $\hat{E}_m f\leq 0$ for $f(X)=\ell_{\hm}(X) - \ell_{h^*_N}(X)$.
\hfill$\blacksquare$

%\begin{remark}\label{remark:Important}
%We note that, in order to bound~$V_r$, we could have used Theorem \ref{thm:SupConc2} in Step 2 of the proof.
%In this case the term $\E[V_r]$, which is an expectation for sampling without replacement, would have been  
%replaced by an expectation for sampling with replacement (like in Theorem \ref{thm:eb2}).
%Repeating all the remaining steps, we obtain the same result as in Theorem \ref{thm:LocExBound} 
%with the factor $N/m^2$ replaced by $1/m$ and slightly worse constants.
%For completeness, we present this modified result as Theorem~\ref{thm:LocExBound_Mod} in Appendix~\ref{sec:Proofs2} 
%of the supplemental material, together with its detailed proof.
%\end{remark}

\medskip
We finally present excess risk bounds for $\Ex_u(\hm)$. The first one is based on Theorem \ref{thm:LocExBound}:
\begin{corollary}
\label{corr:MainResult}
Under the assumptions of Theorem \ref{thm:LocExBound},
for any $t>0$, with probability greater than $1-2e^{-t}$, we have:
\vspace{-5pt}
\[
\Ex_u(\hm) \leq \frac{N}{u}\left( 51\frac{r_m^*}{B} + 17Bt \frac{N}{m^2}\right) + \frac{N}{m}\left( 51 \frac{r_u^*}{B} + 17Bt \frac{N}{u^2}\right).
\]
\end{corollary}
The following version is based on Theorem \ref{thm:LocExBound_Mod} and replaces the factors $N/m^2$ and $N/u^2$ appearing in the previous excess risk bound by $1/m$ and $1/u$, respectively:
\begin{corollary}
\label{corr:MainResult_Mod}
Under the assumptions of Theorem \ref{thm:LocExBound_Mod},
for any $t>0$, with probability greater than $1-2e^{-t}$, we have:
\[
\Ex_u(\hm) \leq \frac{N}{u}
\left( 901\frac{K}{B} r_m^* + \frac{t(16 + 25B)}{3m}\right) 
+ 
\frac{N}{m}
\left( 901\frac{K}{B} r_u^* + \frac{t(16 + 25B)}{3u}\right).
\]
\end{corollary}

\medskip\noindent {\bf Proof sketch} ~
%This results 
Corollary \ref{corr:MainResult} can be proved by noticing that $h^*_u$ is an empirical risk minimizer (similar to $\hat{h}_m$, but computed on the test set).
Thus, repeating the proof of Theorem \ref{thm:LocExBound}, we immediately obtain the same bound for $h^*_u$ as in 
Theorem \ref{thm:LocExBound} with $r^*_m$ and $N/m^2$ replaced by $r^*_u$ and $N/u^2$, respectively.
This shows that the overall errors $\LN(\hm)$ and $\LN(h^*_u)$ are close to each other.
It remains to apply an intermediate step, obtained in the proof of Theorem \ref{thm:LocExBound}.
Corollary \ref{corr:MainResult_Mod} is proved in a similar way.
%We present the detailed proof of Corollary~\ref{corr:MainResult} in Appendix~\ref{sec:Proofs2}.
The detailed proofs are presented in Appendix~\ref{sec:Proofs2}.
\hfill$\blacksquare$\medskip

In order to get a more concrete grasp of the key quantities $r^*_m$ and $r^*_u$
in Corollary \ref{corr:MainResult_Mod}, we can directly apply the
machinery developed in the inductive case by \cite{BBM05} to get an
upper bound.  For concreteness, we consider below the case of a kernel class.
Observe that, by an application of Corollary
\ref{corr:ExWandWO} to the left-hand side of \eqref{eq:subRoor},
the bounds below for the inductive $r^*_m,r^*_u$ of 
Corollary~\ref{corr:MainResult_Mod} are valid as well 
for their transductive siblings of Corollary \ref{corr:MainResult}; 
though by doing so we lose essentially 
any potential advantage (apart from tighter multiplicative constants) of
using Theorem~\ref{thm:LocExBound}/Corollary~\ref{corr:MainResult} over 
Theorem~\ref{thm:LocExBound_Mod}/Corollary~\ref{corr:MainResult_Mod}.
As  pointed out in Remark \ref{remark:WhyBad}, the regime of sampling without
replacement could 
%be more preferable.
%Corollary~\ref{corr:MainResult} can thus 
lead potentially to an improved bound (at least when $m=\Omega(N)$).
Whether it is possible to take advantage of this fact and 
develop tighter bounds specifically for the fixed point of \eqref{eq:subRoor}
is an open question and left for future work.

\begin{corollary}
\label{cor:kern}
Let $k$ be a positive semidefinite kernel on $\mathcal{X}$ with
$\sup_{x \in \mathcal{X}}  k(x,x)\leq 1$, and $\C_k$ the associated reproducing 
kernel Hilbert space.
Let $\mathcal{H}:=\{f \in \C_k: \|f\|\leq 1\}$, and $\F^*$ the associated
excess loss class. Suppose that Assumptions~\ref{ass:Excess} are satisfied
and assume moreover that the loss function $\ell$ is $L$-Lipschitz in its
first variable
and also that $E\bigl(h(X) - h^*(X)\bigr)^2 \leq B \bigl(L(h) - L(h^*)\bigr)$ for all $h\in\Hyp$.
Let $K_N$ be the $N \times N$ normalized kernel Gram matrix with
entries $(K_N)_{ij} := \frac{1}{N}k(X_i,X_j)$, where
${\bf X}_N=(X_1,\ldots,X_N)$; denote $\lambda_{1,N} \geq \ldots \geq
\lambda_{N,N}$ its ordered eigenvalues. Then, for $k=u$ or $k=m$:
\vspace{-0.13cm}
\[ 
r^*_k \leq c_L \min_{0\leq \theta \leq k} \left( \frac{\theta}{k} + 
\sqrt{\frac{1}{k}\sum_{ i\geq \theta} \lambda_{i,N}} \right)\,, \quad
%r^*_u \leq c_L \min_{0\leq \theta \leq u} \left( \frac{\theta}{u} + 
%\sqrt{\frac{1}{u}\sum_{ i\geq \theta} \lambda_{i,N}} \right)\,,
\]
\vspace{-0.18cm}
where $c_L$ is a constant depending only on $L$.
\end{corollary}
This result is obtained as a direct application of the results of \citealp{BBM05},
Section 6.3; \citealp{Men2003}, the only important point being that the generating
distribution %for the data in the inductive case is now replaced by the 
is the uniform
distribution on ${\bf X}_N$. Similar to the discussion there, we note that
$r^*_m$ and $r^*_u$ are at most of order $1/\sqrt{m}$ and $1/\sqrt{u}$, respectively, 
and possibly much smaller if the eigenvalues have a fast decay .

\begin{remark}
The question of {\em transductive convergence rates} is somewhat delicate,
since all results stated here assume a fixed set ${\bf X}_N$, as reflected
for instance in the bound of Corollary \ref{cor:kern} depending on the
eigenvalues of the kernel Gram matrix of the set ${\bf X}_N$. In order
to give a precise meaning to {\em rates}, one has to specify
how ${\bf X}_N$ evolves as $N$ grows. A natural setting for this is
\cite{Vap98}'s second transductive setting where ${\bf X}_N$ is i.i.d. from
some generating distribution. In that case we think it is possible to
adapt once again the results of \cite{BBM05} in order to relate the quantities
$r^*_m(N)$ to asymptotic counterparts as $N\rightarrow \infty$, though
we do not pursue this avenue in the present work.
\end{remark}

%======================================================================================================================
\section{Conclusion}
%======================================================================================================================

In this paper, we have considered the setting of transductive learning over a broad class of bounded and nonnegative loss functions. 
We provide excess risk bounds for the transductive learning setting based on the localized complexity of the hypothesis class,
which hold under general assumptions on the loss function and the hypothesis class.
When applied to kernel classes, the transductive excess risk bound can be formulated in terms of the tailsum of the eigenvalues of the kernels,
similar to the best known estimates in inductive learning.
The localized excess risk bound is achieved by proving two novel and very general \emph{concentration inequalities} for suprema of empirical processes 
when sampling without replacement, which are of potential interest also in various other application areas in machine learning and 
learning theory, where they may serve as a fundamental mathematical tool.

For instance, sampling without replacement is commonly employed in the Nystr\"om method \citep{Kumar2012}, which is an efficient technique 
to generate low-rank matrix approximations in large-scale machine learning.
Another potential application area of our novel concentration inequalities could be the analysis of randomized sequential algorithms such as stochastic
gradient descent and randomized coordinate descent, practical implementations of which often deploy sampling without replacement \citep{RechtR12}.
Very interesting  also would be to explore whether the proposed techniques could be used to generalize matrix Bernstein inequalities \citep{Tropp12} to 
the case of sampling without replacement, which could be used to analyze matrix completion problems \citep{KoltchinskiiLT11}.
The investigation of application areas beyond the transductive learning setting is, however, outside of the scope of the present paper.

%======================================================================================================================
\acks{The authors are thankful to Sergey Bobkov, Stanislav Minsker, and Mehryar Mohri for stimulating discussions and to the anonymous reviewers for their helpful comments. Marius Kloft acknowledges a postdoctoral fellowship by the German Research Foundation (DFG).}
%======================================================================================================================

%======================================================================================================================
\bibliographystyle{abbrvnat}  % We commented out the plainnat style in the jmlr style file!!
\bibliography{TolstikhinBib}
%======================================================================================================================

%======================================================================================================================
\appendix
%======================================================================================================================

%======================================================================================================================
\section{Bousquet's version of Talagrand's concentration inequality}\label{appendix:Bousquet}
%======================================================================================================================

Here we use the setting and notations of Section \ref{sec:Concentr}.
\begin{theorem}[\cite{Bphd02}]
\label{thm:Bphd02}
Let $v = m\sigma^2 + 2\E[Q_m]$ and for $u> -1$ let $\phi(u) = e^u - u - 1$,\: $h(u) = (1+u)\log(1+u) - u$.
Then for any $\lambda \geq 0$ the following upper bound on the moment generating function holds:
\begin{equation}
\label{eq:BousquetMGF}
\E\left[e^{\lambda(Q_m - \E[Q_m])}\right] \leq e^{v \phi(\lambda)}.
\end{equation}
We also have for any $\epsilon\geq 0$:
\begin{equation}
\label{eq:BousquetConcentr_1}
\Prob\left\{
Q_m-\E[Q_m] \geq \epsilon
\right\}
\leq
\exp\left\{-vh\left(\frac{\epsilon}{v}\right)\right\}.
\end{equation}
Noting that $h(u) \geq \frac{u^2}{2(1+u/3)}$ for $u>0$, one can derive the following more illustrative version:
\begin{equation}
\label{eq:BousquetConcentr_2}
\Prob\left\{
Q_m-\E[Q_m] \geq \epsilon
\right\}
\leq
\exp\left\{
-\frac{\epsilon^2}{2(v + \epsilon/3)}
\right\}.
\end{equation}
Also for all $t\geq 0$ the following holds with probability greater than $1-e^{-t}$:
\begin{equation}
\label{eq:BousquetDev}
Q_m \leq \E[Q_m] + \sqrt{2vt} + \frac{t}{3}.
\end{equation}
\end{theorem}
Note that \cite{B02} provides similar bounds for $Q_m = \sup_{f\in\F} \left|\sum_{i=1}^m f(X_i)\right|$.

%======================================================================================================================
\section{Proofs from Section \ref{sec:Concentr}}\label{appendix:Proofs}
%======================================================================================================================

First we are going to prove Theorem \ref{thm:SupConc2}, which is a direct consequence of Bousquet's inequality of Theorem \ref{thm:Bphd02}. It is based on the following \emph{reduction theorem} due to \cite{Hoe63}:
\begin{theorem}[\cite{Hoe63}]
\footnote{Hoeffding initially stated this result only for real valued random variables.
However all the steps of proof hold also for vector-valued random variables.
For the reference see Section D of \cite{GN10}.}
\label{thm:hoeffTrick}
Let $\{U_1,\dots,U_m\}$ and $\{W_1,\dots,W_m\}$ be sampled uniformly from a finite set of $d$-dimensional 
vectors $\{\vec v_1,\dots,\vec v_N\}\subset \R^d$ with and without replacement, respectively.
Then, for any continuous and convex function ${F\colon \R^d \to \R}$, the following holds:
\[
\E\left[ F\left(\sum_{i=1}^m W_i\right)\right] \leq \E\left[F\left(\sum_{i=1}^m U_i\right)\right].
\]
\end{theorem}
Also we will need the following technical lemma:
\begin{lemma}
\label{lemma:FConvex}
Let $\vec x = (x_1,\dots,x_d)^{\mathrm{T}}\in\R^d$.
Then the following function is convex for all $\lambda > 0$ 
\[
F(\vec x) = \exp\left(
\lambda \sup_{i=1,\dots,d}x_i
\right).
\]
\end{lemma}
\begin{proof}
Let us show that, if $g\colon\R\to\R$ is a convex and nondecreasing function and $f\colon \R^d\to\R$ is convex, 
then $g\bigl(f(\vec x)\bigr)$ is also convex.
Indeed, for $\alpha\in[0,1]$ and $\vec x', \vec x'' \in \R^d$:
\[
g\Bigl( f\bigl( \alpha \vec x' + (1-\alpha) \vec x''\bigr) \Bigr)
\leq
g\bigl( \alpha f(\vec x') + (1-\alpha) f(\vec x'') \bigr)
\leq
\alpha g\bigl(f(\vec x')\bigr)
+ (1-\alpha)\alpha g\bigl(f(\vec x'')\bigr).
\]
Considering the fact that $g(y) = e^{\lambda y}$ is convex and increasing for $\lambda > 0$, it remains to show that $f(\vec x) = \sup_{i=1,\dots,d}(x_i)$ is convex.
For all $\alpha\in[0,1]$ and $\vec x',\vec x''\in\R^d$, the following holds:
\[
\sup_{i=1,\dots, d}\bigl( \alpha x'_i + (1-\alpha) x''_i\bigr)
\leq
\alpha \sup_{i=1,\dots, d} x'_i + 
(1-\alpha) \sup_{i=1,\dots, d} x''_i,
\]
which concludes the proof.
\end{proof}

\medskip
We will proove Theorem \ref{thm:SupConc2} for a finite class of functions $\F = \{f_1,\dots, f_M\}$.
The result for uncountable case follows by taking a limit of a sequence of finite sets.

\emph{Proof of Theorem \ref{thm:SupConc2}:}
Let $\{U_1,\dots,U_m\}$ and $\{W_1,\dots,W_m\}$ be sampled uniformly from a finite set of $M$-dimensional vectors $\{\vec v_1,\dots,\vec v_N\}\subset \R^M$ with and without replacement respectively, where
$\vec v_j = \bigl( f_1(c_j), \dots, f_M(c_j)\bigr)^{\mathrm{T}}$.
Using Lemma \ref{lemma:FConvex} and Theorem \ref{thm:hoeffTrick}, we get that for all $\lambda > 0$:
\begin{equation}
\label{eq:TalProof1}
\E\left[ e^{\lambda Q'_m}\right]
=
\E\left[
\exp\left(
\lambda \sup_{j=1,\dots,M} \left(
\sum_{i=1}^m W_i\right)_j
\right)
\right]
\leq
\E\left[
\exp\left(
\lambda \sup_{j=1,\dots,M} \left(
\sum_{i=1}^m U_i\right)_j
\right)
\right]
=\E\left[ e^{\lambda Q_m}\right],
\end{equation}
where the lower index $j$ indicates the $j$-th coordinate of a vector.
Using the upper bound \eqref{eq:BousquetMGF} on the moment generating function of $Q_m$ provided by  Theorem \ref{thm:Bphd02}, we proceed as follows:
\[
\E\left[ e^{\lambda Q'_m}\right]
\leq
\E\left[ e^{\lambda Q_m}\right]
\leq
e^{\lambda \E[Q_m] + v \phi(\lambda)},
\]
or, equivalently,
\[
\E\left[ e^{\lambda (Q'_m - \E[Q'_m])}\right]
\leq
e^{\lambda (\E[Q_m] - \E[Q'_m]) + v \phi(\lambda)}.
\]
Using Chernoff's method, we obtain for all $\epsilon \geq 0$ and $\lambda > 0$:
\begin{equation}
\label{eq:TalChernoff}
\Prob\left\{
Q'_m - \E[Q'_m] \geq \epsilon
\right\}
\leq 
\frac{\E\left[e^{\lambda(Q'_m - \E[Q'_m])}\right]}{e^{\lambda \epsilon}} 
\leq 
\exp\bigl(\lambda (\E[Q_m] - \E[Q'_m]) + v \phi(\lambda)- \lambda \epsilon\bigr).
\end{equation}
The term on the right-hand side of the last inequality achieves its minimum for
\begin{equation}
\label{eq:lambdaOpt}
\lambda = \log\left(
\frac{v + \epsilon - \E[Q_m] + \E[Q'_m]}{v}
\right).
\end{equation}
Thus we have the technical condition $\epsilon\geq \E[Q_m] - \E[Q'_m]$.
Otherwise we set $\lambda = 0$ and obtain a trivial bound equal to 1.
The inequality $\E[Q_m] \geq \E[Q'_m]$ also follows from Theorem~\ref{thm:hoeffTrick} by exploiting the fact that the supremum is a convex function
(which we showed in the proof of Lemma \ref{lemma:FConvex}).
Inserting \eqref{eq:lambdaOpt} into \eqref{eq:TalChernoff}, we obtain the first inequality of Theorem \ref{thm:SupConc2}. 
The second inequality follows from observing that $h(u) \geq \frac{u^2}{2(1+u/3)}$ for $u>0$.
The deviation inequality can then be obtained using standard calculus.
For details we refer to Section~2.7.2 of \cite{Bphd02}.
\hfill$\blacksquare$.

\begin{remark}
It should be noted that \cite{KR05} derive an upper bound on
$\E\left[e^{-\lambda(Q_m - \E[Q_m])}\right] $ for $\lambda \geq 0$.
This upper bound on the moment generating function together with Chernoff's method leads to an upper bound on $\Prob\left\{\E[Q_m] - Q_m \geq \epsilon\right\}$.
However, the proof technique used in Theorem~\ref{thm:SupConc2} cannot be used in this case, since Lemma \ref{lemma:FConvex} does not hold for negative $\lambda$.
\end{remark}

\bigskip
\emph{Proof of Lemma \ref{lemma:expect}:}
We have already proved the first inequality of the lemma.
Regarding the second one, using the definitions of $\E[Q_m]$ and $\E[Q'_m]$, we have:
\[
\E[Q_m] - \E[Q'_m]
=
\frac{1}{N^m}\sum_{x_1,\dots,x_m}\sup_{f\in\mathcal{F}}\sum_{i=1}^m f(x_i)
+
\left(\frac{1}{N^m} - \frac{(N-m)!}{N!}\right)
\sum_{z_1,\dots,z_m}\sup_{f\in\mathcal{F}}\sum_{i=1}^m f(z_i),
\]
where the first sum is over all ordered sequences $\{x_1,\dots, x_m\}\subset\mathcal{C}$ containing duplicate elements, 
while the second sum is over all ordered sequences $\{z_1,\dots, z_m\}\subset\mathcal{C}$ with no duplicates.
Note that the second sum has exactly $m! \cdot C^m_N$ summands, which means that the first one has $N^m - m!\cdot C^m_N$ summands.
Considering the fact that $\frac{1}{N^m} \leq \frac{(N-m)!}{N!}$ and $f(x)\in[-1,1]$ for all $x\in\mathcal{C}$, we obtain:
\begin{align*}
\E[Q_m] - \E[Q'_m]
&\leq
m\left(\frac{N^m - m! \cdot C^m_N}{N^m}\right)
+
m\left(\frac{(N-m)!}{N!} - \frac{1}{N^m}\right)m!\cdot C^m_N\\
&=
2m\left(\frac{N^m - m! \cdot C^m_N}{N^m}\right)\\
&=
2m -2 m\left( 1\cdot \left(1 - \frac{1}{N}\right)\cdots \left(1 - \frac{m-1}{N}\right) \right)\\
&\leq
2m - 2m\left(1 - \frac{m-1}{N}\right)^m\\
&=
2m\left(\frac{m-1}{N}\right)\left(1 + \left(1 - \frac{m-1}{N}\right) + \dots + \left(1 - \frac{m-1}{N}\right)^{m-1}\right)\\
&\leq
2m\left(\frac{m-1}{N}\right)m\\
&\leq
2\frac{m^3}{N},
\end{align*}
which was to show.
\hfill$\blacksquare$

\bigskip
\label{ProofBegin}
In order to prove Theorem \ref{thm:SupConc1}, we need to state the result presented in Theorem 2.1 of \cite{B04} and derive its slightly modified version. From now on we will follow the presentation in \cite{B04}.

Let us consider the following subset of discrete cube, which we call \emph{the slice}:
\[
\mathcal{D}_{n,k} = \bigl\{ x = (x_1,\dots, x_n)\in\{0,1\}^n \colon x_1 + \dots + x_n = k\bigr\}.
\]
Neighbors are points that differ exactly in two coordinates. 
Thus every point $x\in\mathcal{D}_{n,k}$ has exactly $k(n-k)$ neighbours $\{s_{ij} x\}_{i\in I(x), j\in J(x)}$, where
\[
I(x) = \{i\leq n \colon x_i = 1\},
\quad
J(x) = \{j\leq n \colon x_j = 0\},
\]
and $(s_{ij} x)_r = x_r$ for $r\neq i,j$, $(s_{ij} x)_i = x_j$, $(s_{ij} x)_j = x_i$.
For any function $g$ defined on $\mathcal{D}_{n,k}$ and $x\in\mathcal{D}_{n,k}$, let us introduce the following quantity:
\[
V^g(x) = \sum_{i\in I(x)}\sum_{j \in J(x)}\bigl(g(x) - g(s_{ij}x)\bigr)^2,
\]
which can be viewed as the Euclidean length of the discrete gradient $|\nabla g(x)|^2$ of the function $g$.

The following result can be found in Theorem 2.1 of \cite{B04}:
\begin{theorem}[\cite{B04}]
\label{thm:Bobkov}
Consider the real-valued function $g$ defined on $\mathcal{D}_{n,k}$ and the uniform distribution $\mu$ over the set $\mathcal{D}_{n,k}$.
Assume there is a constant $\Sigma^2\geq0$ such that $V^g(x)\leq \Sigma^2$ for all $x$.
Then for all $\epsilon \geq 0$:
\[
\mu\{g(x) - \E[g(x)] \geq \epsilon\} \leq \exp\left\{-\frac{(n+2)\epsilon^2}{4\Sigma^2}\right\}.
\]
The same upper bound also holds for $\mu\{\E[g(x)] - g(x) \geq \epsilon\}$.
\end{theorem}

Using the notations of Section \ref{sec:Concentr}, we define the following function $g\colon \mathcal{D}_{N,m} \to \R$:
\begin{equation}
\label{eq:funcID}
g(x) = \sup_{f\in\F} \sum_{i\in I(x)} f(c_i).
\end{equation}
Note that, if $x$ is distributed uniformly over the set $\mathcal{D}_{N,m}$, the random variables $Q'_m$ and $g(x)$ are identically distributed.
Thus we thus can use Theorem \ref{thm:Bobkov} to derive concentration inequalities for $Q'_m$.
However, it is not trivial to bound the quantity $V^g(x)$. Instead we define the following quantity, related to $V^g(x)$:
\[
V^g_+(x)
=
\sum_{i\in I(x)}\sum_{j \in J(x)}\bigl(g(x) - g(s_{ij}x)\bigr)^2\mathbbm{1}\{g(x) \geq g(s_{ij}x)\},
\]
where $\mathbbm{1}\{A\}$ is an indicator function.
Now we state the following modified version of Theorem \ref{thm:Bobkov}:
\begin{theorem}
\label{thm:BobkovMod}
Consider a real-valued function $g$ defined on $\mathcal{D}_{n,k}$ and the uniform distribution $\mu$ over the set $\mathcal{D}_{n,k}$.
Assume there is a constant $\Sigma^2\geq 0$ such that $V_+^g(x)\leq \Sigma^2$ for all $x$.
Then for all $\epsilon \geq 0$:
\[
\mu\{g(x) - \E[g(x)] \geq \epsilon\} \leq \exp\left\{-\frac{(n+2)\epsilon^2}{8\Sigma^2}\right\}.
\]
The same upper bound also holds for $\mu\{\E[g(x)] - g(x) \geq \epsilon\}$.
\end{theorem}
\begin{proof}
We are going to follow the steps of the proof of Theorem \ref{thm:Bobkov}, presented in \cite{B04}.
The author shows that, for any real-valued function $g$ defined on $\mathcal{D}_{n,k}$, the following holds:
\begin{gather}
\notag
(n+2) \bigl( \E\bigl[ e^{g(x)} \log e^{g(x)}\bigr] - \E\bigl[e^{g(x)}\bigr] \E\bigl[\log e^{g(x)}\bigr] \bigr)
\\
\label{eq:BobkovProofStart}
{}\leq
\E\left[
\sum_{i\in I(x)}\sum_{j \in J(x)}
\bigl( g(x) - g(s_{ij}x)\bigr)\bigl( e^{g(x)} - e^{g(s_{ij}x)}\bigr)
\right].
\end{gather}
Note that for any $a,b\in\R$:
\begin{equation}
\label{eq:expTmp}
(a - b)(e^a - e^b)\leq \frac{e^a + e^b}{2}(a-b)^2.
\end{equation}
We can re-write the right-hand side of inequality \eqref{eq:BobkovProofStart} in the following way:
\begin{align*}
&\E\left[
\sum_{i\in I(x)}\sum_{j \in J(x)}
\bigl( g(x) - g(s_{ij}x)\bigr)\bigl( e^{g(x)} - e^{g(s_{ij}x)}\bigr)
\right]\\
&{}=
2\cdot
\E\left[
\sum_{i\in I(x)}\sum_{j \in J(x)}
\bigl( g(x) - g(s_{ij}x)\bigr)\bigl( e^{g(x)} - e^{g(s_{ij}x)}\bigr)\mathbbm{1}\{g(x) \geq g(s_{ij}x)\}\right].
\end{align*}
Using \eqref{eq:expTmp}, we get:
\begin{align*}
&\E\left[
\sum_{i\in I(x)}\sum_{j \in J(x)}
\bigl( g(x) - g(s_{ij}x)\bigr)\bigl( e^{g(x)} - e^{g(s_{ij}x)}\bigr)
\right]
\\
&{}\leq
2 \cdot \E\left[
\sum_{i\in I(x)}\sum_{j \in J(x)}
\frac{\bigl(e^{g(x)} + e^{g(s_{ij}x)}\bigr)}{2}\bigl( g(x) - g(s_{ij}x)\bigr)^2\mathbbm{1}\{g(x) \geq g(s_{ij}x)\}\right]\\
&{}\leq
2\cdot \E\left[
\sum_{i\in I(x)}\sum_{j \in J(x)}
e^{g(x)}\bigl( g(x) - g(s_{ij}x)\bigr)^2\mathbbm{1}\{g(x) \geq g(s_{ij}x)\}\right]\\
&= 2\E\left[V^g_+(x)e^{g(x)}\right].
\end{align*}
Thus we obtain the following inequality:
\[
(n+2) \bigl( \E\bigl[ e^{g(x)} \log e^{g(x)}\bigr] - \E\bigl[e^{g(x)}\bigr] \E\bigl[\log e^{g(x)}\bigr] \bigr)
\leq
2\E\left[V^g_+(x)e^{g(x)}\right].
\]
Applying this inequality to $\lambda g$, where $\lambda\in\R$, we get:
\begin{equation}
\label{eq:BobkovProofM}
(n+2) \Bigl( \E\bigl[ e^{\lambda g(x)} \log e^{\lambda g(x)}\bigr] - \E\bigl[e^{\lambda g(x)}\bigr] \E\bigl[\log e^{\lambda g(x)}\bigr] \Bigr)
\leq 
2\E\left[V^{\lambda g}_+(x)e^{\lambda g(x)}\right]
\leq
2\Sigma^2\lambda^2\E\left[e^{\lambda g(x)}\right].
\end{equation}
As mentioned in the proof of Theorem \ref{thm:Bobkov} in \cite{B04}, inequality \eqref{eq:BobkovProofM} implies the following upper bound on the moment generating function:
\begin{equation}
\label{eq:BobkovLaplace}
\E\left[e^{\lambda(g(x) - \E[g(x)])}\right] \leq e^{\frac{2\Sigma^2\lambda^2}{n+2}}.
\end{equation}
This fact is known as the \emph{Herbst argument} and plays an important role in the entropy method \citep{BLM13}.
Now we apply Chernoff's method, which gives us for all $\lambda, \epsilon \geq 0$:
\[
\mu\left\{
g(x) - \E[g(x)] \geq \epsilon
\right\}
\leq 
\frac{\E\left[e^{\lambda(g(x) - \E[g(x)])}\right]}{e^{\lambda \epsilon}} 
\leq 
e^{\frac{2\Sigma^2\lambda^2}{n+2} - \lambda\epsilon}.
\]
We conclude the proof by choosing $\lambda = \frac{\epsilon(n+2)}{4\Sigma^2}$.

An upper bound for $\mu\{\E[g(x)] - g(x) \geq \epsilon\}$ can be obtained using \eqref{eq:BobkovLaplace} for $\lambda < 0$:
\[
\mu\left\{
\E[g(x)] - g(x) \geq \epsilon
\right\}
=
\mu\left\{
\lambda\bigl(g(x) - \E[g(x)] \bigr)\geq -\lambda\epsilon
\right\}
\leq
\frac{\E\left[e^{\lambda(g(x) - \E[g(x)])}\right]}{e^{-\lambda \epsilon}} 
\leq
e^{\frac{2\Sigma^2\lambda^2}{n+2} + \lambda\epsilon}.
\]
Now it remains to choose $\lambda = -\frac{\epsilon(n+2)}{4\Sigma^2}$.
\end{proof}

We will need the following technical result:
\begin{lemma}\label{lemma:tech}
For any sequence of real numbers $\{x_1,\dots,x_n\}$ the following holds:
\[
\frac{1}{n^2}\sum_{1\leq i<j \leq n}(x_i-x_j)^2 = 
\frac{1}{n}\sum_{i=1}^n\Biggl(x_i - \frac{1}{n}\sum_{j=1}^n x_j\Biggr)^2.
\]
\end{lemma}
\begin{proof}
Notice that it holds:
\begin{align*}
\frac{1}{n}\sum_{i=1}^n\Biggl(x_i - \frac{1}{n}\sum_{j=1}^n x_j\Biggr)^2 
&= \frac{1}{n}\sum_{i=1}^n\left(x_i^2 -\frac{2}{n}x_i\sum_{j=1}^n x_j + \frac{1}{n^2}\Biggl(\sum_{j=1}^nx_j\Biggr)^2\right)\\
&=\frac{1}{n}\left(\sum_{i=1}^n x_i^2 -\frac{2}{n}\sum_{i=1}^n x_i\sum_{j=1}^n x_j + \frac{1}{n^2}\sum_{i=1}^n\Biggl(\sum_{j=1}^nx_j\Biggr)^2\right)\\
&=
\frac{1}{n}\left(\sum_{i=1}^nx_i^2 -\frac{1}{n}\Biggl(\sum_{j=1}^n x_j\Biggr)^2 \right)\\
&=\frac{1}{n^2}\Biggl((n-1)\sum_{i=1}^nx_i^2 - 2\!\!\!\sum_{1\leq i < j \leq n} x_i x_j \Biggr)\\
&=\frac{1}{n^2}\sum_{1\leq i<j \leq n}(x_i-x_j)^2.
\end{align*}
\end{proof}

Now we are ready to prove Theorem \ref{thm:SupConc1}.

\medskip
\emph{Proof of Theorem \ref{thm:SupConc1}:}
We will apply Theorem \ref{thm:BobkovMod} for the function $g(x)$ defined in \eqref{eq:funcID}, where $x$ is distributed uniformly over $\mathcal{D}_{N,m}$.
As noted above this will lead to a concentration inequality for $Q'_m$, since $Q'_m$ and $g(x)$ are distributed identically.
Hence, all we need is to obtain an upper bound on $V^g_+(x)$.

To this end, let us consider two functions $g_1,g_2\colon \mathcal{A}\to\mathbb{R}$, defined on some set $\mathcal{A}$.
Assume that ${\sup_{a\in\mathcal{A}} g_1(a) = g_1(\bar{a})}$ for some $\bar{a}\in\mathcal{A}$.
Then the following holds:
\begin{equation}
\label{eq:Tool1}
\left(
\sup_{a\in\mathcal{A}} g_1(a)
-
\sup_{a\in\mathcal{A}} g_2(a)
\right)^2
\mathbbm{1}\left\{\sup_{a\in\mathcal{A}} g_1(a)
\geq
\sup_{a\in\mathcal{A}} g_2(a)\right\}
\leq
\bigl(
g_1(\bar{a})
- g_2(\bar{a})
\bigr)^2.
\end{equation}

Let us assume that, for $x\in\mathcal{D}_{N,m}$, the supremum in the definition \eqref{eq:funcID} of $g(x)$ is achieved for $\bar{f}\in\mathcal{F}$.
Then we have:
\begin{align*}
V^{g}_+(x)
&=
\sum_{i\in I(x)}\sum_{j \in J(x)}
\left(
g(x) - g(s_{ij}x)
\right)^2
\mathbbm{1}\left\{g(x) \geq g(s_{ij}x)\right\}\\
&\leq
\sum_{i\in I(x)}\sum_{j \in J(x)}
\Biggl(
\sum_{k\in I(x)} \bar{f}(c_k)
-
\sum_{k\in I(s_{ij}x)} \bar{f}\bigl(c_k\bigr)
\Biggr)^2\\
&=
\sum_{i\in I(x)}\sum_{j \in J(x)}
\Bigl(
\bar{f}(c_i) - \bar{f}(c_j)
\Bigr)^2\\
&\leq
\sum_{1\leq i < j \leq N}
\Bigl(
\bar{f}(c_i) - \bar{f}(c_j)
\Bigr)^2\\
&=
N^2 \V\bigl[\bar{f}(X_1)\bigr],
\end{align*}
where the first inequality follows from \eqref{eq:Tool1}
and the second inequality follows from Lemma \ref{lemma:tech}.
Now note that, since the function $\bar{f}$ depends on the choice of $x$, the following holds for all $x\in\mathcal{D}_{N,m}$:
\[
V^g_+(x) \leq N^2 \sup_{f\in\mathcal{F}} \V[f(X_1)] = N^2\sigma^2.
\]
We conclude the proof by an application of Theorem \ref{thm:BobkovMod}.
\label{ProofEnd}
\hfill$\blacksquare$

%======================================================================================================================
\section{Further discussions on Section \ref{sec:Concentr}}\label{sec:FurtherDiscussions}
%======================================================================================================================

First we note that the result of Theorem \ref{thm:EP} is uniformly sharper than what could have been obtained for~$Q_m$ using McDiarmid's inequality, 
by a factor of $\frac{N-1/2}{N-m}$ (fraction of the training sample) in the exponent.
This suggests that when sampling without replacement things are more concentrated than when sampling with replacement.
This general phenomenon is pointed out by several authors: \cite{Serf74} obtains a refinement of Hoeffding's inequality, \cite{EP09} improves McDiarmid's inequality, and \cite{BM13} improve Bennet's and Bernstein's inequalities in the same way for sampling without replacement---opposed 
to the fact that the results of Theorems~\ref{thm:SupConc1}~and~\ref{thm:SupConc2} unfortunately do not improve the known analogues for $Q_m$.
This drawback can possibly be overcome by a more detailed analysis.
This direction is left for the future work.

%======================================================================================================================
\section{Proofs for Section \ref{sec:fastRates}}\label{sec:Proofs2}
%======================================================================================================================

\emph{Proof of Theorem \ref{thm:eb2}:}
Applying Theorem \ref{thm:SupConc2}, we get that with probability greater than $1 - e^{-t}$:
\[
\sup_{h\in\Hyp} \bigl( \LN(h) - \Lm(h) \bigr) \leq 
E_m
+
\sqrt{2\left(\sigma^2_{\Hyp} + 2E_m\right)\frac{t}{m}} + \frac{t}{3m},
\]
which can be further simplified using $\sqrt{a+b}\leq\sqrt{a}+\sqrt{b}$ and then $\sqrt{ab}\leq\frac{a+b}{2}$ in the following way:
\begin{align*}
\sup_{h\in\Hyp} \bigl( \LN(h) - \Lm(h) \bigr) &\leq 
E_m + \sqrt{\frac{2\sigma^2_{\Hyp} t}{m}} + 2\sqrt{\frac{E_mt}{m}} + \frac{t}{3m}\\
&\leq
2E_m + \sqrt{\frac{2\sigma^2_{\Hyp} t}{m}} + \frac{4t}{3m}.
\end{align*}
\hfill$\blacksquare$

\medskip
The proof of Theorem \ref{thm:LocExBound} is based on the following intermediate result appearing in the proof of Theorem 3.3 in \cite{BBM05}.
We state it as a lemma:
\begin{lemma}(Peeling Lemma using Theorem \ref{thm:SupConc1})
\label{lemma:BBM05}
Assume the conditions of Theorem \ref{thm:LocExBound} hold.
Fix some $\lambda > 1$.
For $w(r,f) = \min\{r\lambda^k\colon k\in \mathbb{N}, r\lambda^k \geq E f^2\}$, define the following rescaled version of excess loss class:
\[
\mathcal{G}_r = \left\{
\frac{r}{w(r,f)}f\colon f\in\F^*
\right\}.
\]
Then for any $r>r^*_m$ and $t>0$, with probability greater than $1-e^{-t}$, we have:
\begin{equation}
\label{eq:peelingProof1a}
\sup_{g\in\mathcal{G}_r} Eg - \hat{E}_m g
\leq
\sqrt{r}\left(5\frac{\sqrt{r^*_m}}{B} + 2\sqrt{2t\frac{N}{m^2}}\right).
\end{equation}
\end{lemma}
\begin{proof}
We can repeat exactly the same steps presented in the proof of the first part of Theorem 3.3 of \cite{BBM05} 
(see pages 15--16), but using Theorem \ref{thm:SupConc1} in place of Talagrand's inequality.

Clearly, for any $f\in\F^*$, we have 
\begin{equation}
\label{eq:var1}
\V[f(X)] = E f^2 - (Ef)^2 \leq E f^2.
\end{equation}
Let us now fix some $\lambda > 1$ and $r>0$ and introduce the following rescaled version of excess loss class:
\[
\mathcal{G}_r = \left\{
\frac{r}{w(r,f)}f\colon f\in\F^*
\right\},
\]
where $w(r,f) = \min\{r\lambda^k\colon k\in \mathbb{N}, r\lambda^k \geq E f^2\}$.
Let us consider functions $f\in\F^*$ such that $Ef^2\leq r$, meaning $w(r,f) = r$.
The functions $g\in\mathcal{G}_r$ corresponding to those functions satisfy $g=f$ and thus $\V[g(X)] = \V[f(X)] \leq Ef^2 \leq r$.
Otherwise, if $Ef^2 > r$, then $w(r,f) = \lambda^k r$, and thus the functions $g\in\mathcal{G}_r$ corresponding to them 
satisfy $g=f/\lambda^k$ and $Ef^2\in(r\lambda^{k-1}, r\lambda^k]$.
Thus we have $\V[g] = \V[f]/\lambda^{2k} \leq E f^2 / \lambda^{2k} \leq r$.
We conclude that, for any $g\in \mathcal{G}_r$, it holds $\V[g(x)]\leq r$.

Now we want to upper bound the following quantity:
\[
V_r = \sup_{g\in\mathcal{G}_r} E g - \hat{E}_m g.
\]
Note that any $f\in\F^*$ satisfies $f(X)\in[-1,1]$, and, consequently, all $g\in\mathcal{G}_r$ satisfy $g(X)\in[-1,1]$.
Notice that 
\[
\frac{1}{2}\left(E g - \hat{E}_m g\right) = \frac{1}{m}\sum_{X\in\vec X_m}\frac{E g - g(X)}{2}.
\]
Note that $(E g - g(X))/2\in[-1,1]$ and also $\E\left[E g - g(X)\right] = 0$.
Since $E g$ is not random, using \eqref{eq:var1}, we also have 
\[
\V[(E g - g(X))/2] = \V[g(x)]/4\leq r/4
\]
for all $g\in\mathcal{G}_r$.
We can now apply either Theorem \ref{thm:SupConc1} or Theorem \ref{thm:SupConc2} for the following function class: $\{(Eg - g(X))/2, g\in\mathcal{G}_r\}$.
Here we present the proof based on Theorem \ref{thm:SupConc1}.
Applying it we get that for all $t>0$ with probability greater than $1-e^{-t}$, we have:
\begin{align*}
\frac{1}{2}\sup_{g\in\mathcal{G}_r} E g - \hat{E}_m g
&\leq
\frac{1}{2}\E\left[\sup_{g\in\mathcal{G}_r} E g - \hat{E}_m g\right]
+
2\sqrt{2t\left(\frac{N}{m^2}\right) \frac{1}{4}\sup_{g\in\mathcal{G}_r}\V[g(X)]}\\
&\leq
\frac{1}{2}\E\left[\sup_{g\in\mathcal{G}_r} E g - \hat{E}_m g\right]
+
\sqrt{2t\left(\frac{N}{m^2}\right) r}
\end{align*}
or, rewriting,
\begin{align}
\label{eq:PeelingProof2}
V_r
\leq
\E[V_r] + 2\sqrt{2t\left(\frac{N}{m^2}\right)r}.
\end{align}
Now we set $\F^*(x,y)=\{f\in\F^*\colon x\leq Ef^2 \leq y\}$.
Note that by the assumptions of the theorem, for $f\in\F^*$, we have $\V[f(X)]\leq Ef^2 \leq B \cdot Ef \leq B$.
Define $k$ to be the smallest integer such that $r \lambda^{k+1} \geq B$.
Notice that, for any sets $A$ and $B$, we have:
\[
\E\left[\sup_{g\in A\cup B} Eg - \hat{E}_m g \right]
\leq
\E\left[\sup_{g\in B} Eg - \hat{E}_m g\right]
+
\E\left[\sup_{g\in A} Eg - \hat{E}_m g\right].
\] 
Indeed, since supremum is a convex function, we can use Jensen's inequality to show that each of the terms is positive.
Then we have:
\begin{align*}
\E[V_r] &= \E\left[\sup_{g\in\mathcal{G}_r} E g - \hat{E}_m g\right]\\
& \leq
\E\left[\sup_{f\in\F^*(0,r)} E f - \hat{E}_m f\right]
+
\E\left[\sup_{f\in\F^*(r,B)} \frac{r}{w(r,f)}(E f - \hat{E}_m f)\right]\\
&\leq
\E\left[\sup_{f\in\F^*(0,r)} E f - \hat{E}_m f\right]
+
\sum_{j=0}^k\E\left[\sup_{f\in\F^*(r\lambda^j,r\lambda^{j+1})} \frac{r}{w(r,f)}(E f - \hat{E}_m f)\right]\\
&\leq
\E\left[\sup_{f\in\F^*(0,r)} E f - \hat{E}_m f\right]
+
\sum_{j=0}^k\lambda^{-j}\E\left[\sup_{f\in\F^*(r\lambda^j,r\lambda^{j+1})} (E f - \hat{E}_m f)\right]\\
&\leq
\frac{\psi_m(r)}{B} + \frac{1}{B}\sum_{j=0}^k \lambda^{-j} \psi(r \lambda^{j+1})
\end{align*}
where in the last step we used the assumptions of the theorem.
Now since $\psi_m$ is sub-root, for any $\beta\geq1$, we have $\psi_m(\beta r) \leq \sqrt{\beta} \psi_m(r)$.
Thus
\[
\E[V_r] \leq \frac{\psi_m(r)}{B}\left(1 + \sqrt{\lambda}\sum_{j=0}^k \lambda^{-j/2}\right).
\]
Taking $\lambda = 4$ the r.h.s. is upper bounded by $5\psi_m(r)/B$.
Finally we note that for $r\geq r^*_m$ we have that, for all $r\geq r_m^*$, it holds $\psi_m(r) \leq \sqrt{r/r_m^*}\psi_m(r_m^*) = \sqrt{rr_m^*}$ and thus
\[
\E[V_r] \leq \frac{5}{B}\sqrt{rr^*_m}.
\]
Inserting this upper bound into \eqref{eq:PeelingProof2}, we conclude the proof.
\end{proof}

\emph{Proof of Theorem \ref{thm:LocExBound}:}
Using Lemma \ref{lemma:BBM05}, we obtain that, for any $r>r^*_m$, $t>0$, and $\lambda > 1$, with probability greater than $1-e^{-t}$, we have:
\begin{equation}
\label{eq:peelingProof1}
\sup_{g\in\mathcal{G}_r} Eg - \hat{E}_m g
\leq
\sqrt{r}\left(5\frac{\sqrt{r^*_m}}{B} + 2\sqrt{2t\frac{N}{m^2}}\right),
\end{equation}
where we introduced the rescaled excess loss class:
\[
\mathcal{G}_r = \left\{
\frac{r}{w(r,f)}f\colon f\in\F^*
\right\},
\]
and $w(r,f) = \min\{r\lambda^k\colon k\in \mathbb{N}, r\lambda^k \geq E f^2\}$.
Now we want to chose $r_0>r^*_m$ in such a way that the upper bound of \eqref{eq:peelingProof1} becomes of a form $r_0/(\lambda B K)$.
We achieve this by setting:
\[
r_0 = K^2 \lambda^2 \left(5\sqrt{r^*_m} + 2B\sqrt{2t\frac{N}{m^2}}\right)^2 > r^*_m.
\]
Inserting $r = r_0$ into \eqref{eq:peelingProof1} we obtain:
\begin{equation}
\label{eq:PeelingProof3}
\sup_{g\in\mathcal{G}_{r_0}} Eg - \hat{E}_m g
\leq
\frac{r_0}{\lambda B K}.
\end{equation}
Moreover,  using $(u+v)^2\leq 2(u^2 + v^2)$, we have
\begin{equation}
\label{eq:PeelProof4}
r_0 \leq 50K^2\lambda^2r^*_m + 16 K^2 \lambda^2 B^2 t\left(\frac{N}{m^2}\right).
\end{equation}
Recall that, for any $r>0$ and all $g\in\mathcal{G}_r$, the following holds with probability 1:
\[
E g - \hat{E}_m g \leq \sup_{g\in\mathcal{G}_r} E g - \hat{E}_m g.
\]
Using definition of $\mathcal{G}_r$, we get that, for all $f\in\F^*$, the following holds with probability 1:
\[
E \left(\frac{r}{w(r,f)}f\right) - \hat{E}_m \left(\frac{r}{w(r,f)}f\right) \leq \sup_{g\in\mathcal{G}_r} E g - \hat{E}_m g,
\]
or, equivalently,
\[
E f - \hat{E}_m f \leq  \frac{w(r,f)}{r} \sup_{g\in\mathcal{G}_r} E g - \hat{E}_m g.
\]
Setting $r=r_0$ and using \eqref{eq:PeelingProof3}, we obtain that with probability greater than $1-e^{-t}$
\begin{align*}
\forall f\in\F^*, \forall K > 1:\quad
E f - \hat{E}_m f 
&\leq  
\frac{w(r_0,f)}{r_0} \frac{r_0}{\lambda KB}
=
\frac{w(r_0,f)}{\lambda KB}.
\end{align*}
Now we will use Assumption~\ref{ass:Excess}.2.
If, for $f\in\F^*$, $Ef^2 \leq r_0$ then $w(r_0,f) = r_0$ and using \eqref{eq:PeelProof4} we obtain:
\[
E f - \hat{E}_m f 
\leq
\frac{w(r_0,f)}{\lambda KB}
=
\frac{r_0}{\lambda KB}
\leq
50\frac{K}{B}\lambda r^*_m + 16 \lambda K B t\left(\frac{N}{m^2}\right)
\]
or, rewriting,
\begin{equation}
\label{eq:PeelProof6}
Ef \leq \hat{E}_m f  + 50\frac{K}{B}\lambda r^*_m + 16 \lambda K B t\left(\frac{N}{m^2}\right).
\end{equation}
If otherwise $Ef^2 > r_0$, then $w(r_0, f) = \lambda^i r_0$ for certain value of $i>0$ and also ${E f^2 \in (r_0 \lambda^{i-1}, r_0\lambda^i]}$.
Then we have:
\[
E f - \hat{E}_m f 
\leq
\frac{w(r_0,f)}{\lambda KB}
=
\frac{\lambda^i r_0}{\lambda KB}
=
\frac{\lambda\cdot(\lambda^{i-1} r_0)}{\lambda KB}
\leq
\frac{E f^2}{KB}
\leq
\frac{E f}{K}.
\]
Thus
\begin{equation}
\label{eq:PeelProof7}
E f \leq \frac{K}{K-1} \hat{E}_m f.
\end{equation}
Combining \eqref{eq:PeelProof6} and \eqref{eq:PeelProof7}, we finally get that with probability greater than $1-e^{-t}$
\begin{equation}\label{eq:proofLocExR}
\forall f\in\F^*, \forall K>1:\quad
E f
\leq
\inf_{K > 1}
\frac{K}{K-1} \hat{E}_m f
+
50\frac{K}{B}\lambda r^*_m + 16 \lambda K B t\left(\frac{N}{m^2}\right).
\end{equation}
In the very last step, we recall the definition of $\F^*$ and put $\hat{f}_m = \ell_{\hat{h}_m} - \ell_{h^*_N}$. Notice that
\begin{align*}
\hat{E}_m \hat{f}_m
&=
\hat{E}_m \ell_{\hat{h}_m} - \hat{E}_m \ell_{h^*_N}\\
&=
\Lm(\hat{h}_m) - \Lm(h^*_N) \leq 0,
\end{align*}
while
\[
E \hat{f}_m = \LN(\hat{h}_m) - \LN(h^*_N),
\]
which concludes the proof.
\hfill$\blacksquare$

\bigskip
Let $\{\xi_1,\dots,\xi_m\}$ be random variables sampled \emph{with replacement} from $\vec X_N$.
Denote
\begin{equation}
\label{eq:ErmDef}
E_{r,m} = \E\left[ \sup_{f\in\F^*: E f^2 \leq r} \left(E f - \frac{1}{m}\sum_{i=1}^m f(\xi_i)\right)\right].
\end{equation}
Repeating the proof of peeling Lemma \ref{lemma:BBM05} and using Theorem \ref{thm:SupConc2} instead of Theorem \ref{thm:SupConc1} we immediately obtain the following result:
\begin{lemma}(Peeling Lemma using Theorem \ref{thm:SupConc2})
\label{lemma:BBM05_Mod}
Let $\Hyp$ and $\ell$ be such that Assumptions~\ref{ass:Excess} are satisfied.
Assume there is a sub-root function $\psi_m(r)$ such that 
\[
B\cdot E_{r,m} \leq \psi_m(r),
\]
where $E_{r,m}$ was defined in \eqref{eq:ErmDef}.
Let $r_m^*$ be a fixed point of $\psi_m(r)$.

Fix some $\lambda > 1$.
Then, for any $r>r^*_m$ and $t>0$, with probability greater than $1-e^{-t}$, we have:
\begin{equation}
\label{eq:peelingProof2}
\sup_{g\in\mathcal{G}_r} Eg - \hat{E}_m g
\leq
\sqrt{r}\left(15\frac{\sqrt{r^*_m}}{B} + \sqrt{\frac{2t}{m}}\right) + \frac{8t}{3m},
\end{equation}
where for $w(r,f) = \min\{r\lambda^k\colon k\in \mathbb{N}, r\lambda^k \geq E f^2\}$ we define the following rescaled version of excess loss class:
\[
\mathcal{G}_r = \left\{
\frac{r}{w(r,f)}f\colon f\in\F^*
\right\}.
\]
\end{lemma}

\bigskip
%We present the following variant of Theorem \ref{thm:LocExBound},
%which is now based on the peeling Lemma~\ref{lemma:BBM05_Mod} using Theorem \ref{thm:SupConc2} instead of Theorem \ref{thm:SupConc1}.
\emph{Proof of Theorem \ref{thm:LocExBound_Mod}} is based on the peeling Lemma \ref{lemma:BBM05_Mod} and is similar to the one of Theorem~\ref{thm:LocExBound}.
\medskip

\emph{Proof of Corollary \ref{corr:MainResult}:}
Note that, since $h^*_u$ is also an empirical risk minimizer (computed on the test set), the results of Theorems \ref{thm:LocExBound} 
and \ref{thm:LocExBound_Mod} also hold for $h^*_u$ with every $m$ in the statement replaced by $u$.
Also note that the following holds almost surely:
\begin{align}\label{eq:eq_chain1}
\begin{split}
0 & \leq L_N(\hm) - L_N(h^*_N) \\
&=
L_N(\hm) - L_N(h^*_N) - \Lm(\hm) + \Lm(h^*_N) + \Lm(\hm) - \Lm(h^*_N)\\
&\leq
L_N(\hm) - L_N(h^*_N) - \Lm(\hm) + \Lm(h^*_N)\\
&=
\frac{u}{N}\left(L_u(\hm) - L_u(h^*_N) - \Lm(\hm) + \Lm(h^*_N)\right)
\end{split}
\end{align}
and
\begin{align*}
0 & \leq L_N(h^*_u) - L_N(h^*_N) \\
&=
L_N(h^*_u) - L_N(h^*_N) - L_u(h^*_u) + L_u(h^*_N) + L_u(h^*_u) - L_u(h^*_N)\\
&\leq
L_N(h^*_u) - L_N(h^*_N) - L_u(h^*_u) + L_u(h^*_N)\\
&=
\frac{m}{N}\left(\Lm(h^*_u) - \Lm(h^*_N) - L_u(h^*_u) + L_u(h^*_N)\right),
\end{align*}
where last equations in both cases use the following:
\[
N \cdot L_N(h) = m \cdot \hat{L}_m(h) + u \cdot L_u(h).
\]
Now we are going to use inequality \eqref{eq:proofLocExR} obtained in the proof of Theorem \ref{thm:LocExBound}.
Using the last equation in \eqref{eq:eq_chain1} and, subsequently, employing \eqref{eq:proofLocExR} for $f=\ell_{\hm} - \ell_{h^*_N}$,
where we subtract $\hat{E}_m f$ from both sides of \eqref{eq:proofLocExR}, we obtain:
\begin{align*}
0 &\leq L_u(\hm) - L_u(h^*_N) - \Lm(\hm) + \Lm(h^*_N)\\
& \leq \frac{N}{u}\left( \inf_{K>1}\frac{1}{K-1}\underbrace{\Lm(\hm - h^*_N)}_{\leq 0} + 50\frac{K}{B}\lambda r_m^* + 16\lambda KBt \frac{N}{m^2}\right),
\end{align*}
which holds with probability greater than $1-e^{-t}$.
As noted above the same argument can be used for $h^*_u$, which gives that the following holds:
\begin{align*}
0 &\leq \Lm(h^*_u) - \Lm(h^*_N) - L_u(h^*_u) + L_u(h^*_N)\\
 &\leq \frac{N}{m}\left( \inf_{K>1}\frac{1}{K-1}\underbrace{L_u(h^*_u - h^*_N)}_{\leq 0} + 50\frac{K}{B}\lambda r_u^* + 16\lambda KBt \frac{N}{u^2}\right),
\end{align*}
with probability greater than $1-e^{-t}$.
The union bound gives us that both inequalities hold simultaneously with probability greater than $1-2e^{-t}$.
Or, equivalently,
\[
0 \leq L_u(\hm) - L_u(h^*_N) - \Lm(\hm) + \Lm(h^*_N) \leq \frac{N}{u}\left( 50K\lambda\frac{r_m^*}{B} + 16\lambda KBt \frac{N}{m^2}\right)
\]
and
\[
0 \leq \Lm(h^*_u) - \Lm(h^*_N) - L_u(h^*_u) + L_u(h^*_N) \leq \frac{N}{m}\left( 50K\lambda\frac{r_u^*}{B}  + 16\lambda KBt \frac{N}{u^2}\right).
\]
Summing these two inequalities we obtain
\begin{align*}
0 &\leq L_u(\hm) - L_u(h^*_u) - \Lm(\hm) + \Lm(h^*_u)\\
 &\leq \frac{N}{u}\left( 50\lambda K\frac{r_m^*}{B} + 16\lambda KBt \frac{N}{m^2}\right) + \frac{N}{m}\left( 50\lambda K\frac{r_u^*}{B} + 16\lambda K Bt \frac{N}{u^2}\right).
\end{align*}
Using the fact that $\hm$ and $h^*_u$ are the empirical risk minimizers on the training and test sets, respectively, we finally get:
\begin{align*}
0 &\leq L_u(\hm) - L_u(h^*_u)\\ 
&\leq 
\Lm(\hm) - \Lm(h^*_u) + \frac{N}{u}\left( 50\lambda K\frac{r_m^*}{B} + 16\lambda KBt \frac{N}{m^2}\right) + \frac{N}{m}\left( 50\lambda K\frac{r_u^*}{B} + 16\lambda KBt \frac{N}{u^2}\right)\\
&\leq 
\frac{N}{u}\left( 50\lambda K\frac{r_m^*}{B} + 16\lambda KBt \frac{N}{m^2}\right) + \frac{N}{m}\left( 50\lambda K\frac{r_u^*}{B} + 16\lambda KBt \frac{N}{u^2}\right).
\end{align*}
\hfill$\blacksquare$

\emph{Proof of Corollary \ref{corr:MainResult_Mod}} repeats the same steps using Theorem \ref{thm:LocExBound_Mod} instead of Theorem \ref{thm:LocExBound}.

\bigskip
\noindent We also provide the following auxiliary result:

\begin{corollary}
\label{corr:FullSet}
Under the assumptions of Theorem \ref{thm:LocExBound},
for any $t>0$ and any $K>1$, with probability greater than $1-2e^{-t}$, we have:
\begin{align}
\notag
|\LN(\hm) - \LN(h^*_u)| &\leq \max\left(2K\frac{r_m^*}{B} + 16 KB t \left(\frac{N}{m^2}\right), 2K\frac{r_u^*}{B} + 16 KB t \left(\frac{N}{u^2}\right) \right)\\
\label{eq:LocExU}
&\leq
2K\frac{r_m^* + r_u^*}{B} + 16 KB t N \left(\frac{1}{m^2} + \frac{1}{u^2}\right).
\end{align}
\end{corollary}
\begin{proof}
Notice that $\LN(h^*_u) - \LN(h^*_N)\geq 0$ as well as $\LN(\hm) - \LN(h^*_N)\geq 0$ and then combine Theorem \ref{thm:LocExBound} with its analogue for $h^*_u$ in a union bound.
\end{proof}

\end{document}